\documentclass{article}

\usepackage{natbib}

\usepackage[a4paper, total={6in, 8in}]{geometry}

\usepackage[utf8]{inputenc} 
\usepackage[T1]{fontenc}    
\usepackage{hyperref}       
\usepackage{url}            
\usepackage{booktabs}       
\usepackage{amsfonts}       
\usepackage{nicefrac}       
\usepackage{microtype}      
\usepackage{xcolor}         
\usepackage{graphicx}
\usepackage{dsfont}
\usepackage{subfigure}
\usepackage{amsmath}
\usepackage{amssymb}
\usepackage{mathtools}
\usepackage{amsthm}
\usepackage[capitalize,noabbrev]{cleveref}
\usepackage[textsize=tiny]{todonotes}
\usepackage{kky}
\usepackage{algorithmic}
\usepackage{algorithm}

\renewcommand{\algorithmiccomment}[1]{\bgroup\hfill//~#1\egroup}

\theoremstyle{plain}
\newtheorem{theorem}{Theorem}[section]
\newtheorem{proposition}[theorem]{Proposition}
\newtheorem{lemma}[theorem]{Lemma}

\theoremstyle{definition}

\newtheorem{assumption}[theorem]{Assumption}
\theoremstyle{remark}

\newcommand{\MMAB}{\operatorname{MMAB}}
\newcommand{\support}{\operatorname{support}}

\usepackage{authblk}

\title{Constant or logarithmic regret in asynchronous multiplayer bandits}

\author[1, 3]{Hugo Richard}
\author[2]{Etienne Boursier}
\author[1, 3]{Vianney Perchet}
\affil[1]{ENSAE, Crest, France}
\affil[2]{INRIA, Université Paris Saclay, LMO, Orsay, France}
\affil[3]{Criteo AI Labs, France}

\date{}

\begin{document}

\maketitle

\begin{abstract}
    Multiplayer bandits have recently been extensively studied because of their application to cognitive radio networks. 
    While the literature mostly considers synchronous players, radio networks (e.g. for IoT) tend to have asynchronous devices. This motivates the harder, asynchronous multiplayer bandits problem, which was first tackled with an explore-then-commit (ETC) algorithm  \citep[see][]{dakdoukMassiveMultiplayerMultiarmed2022}, with a regret upper-bound in $\Ocal(T^{\frac{2}{3}})$. Before even considering decentralization, understanding the centralized case was still a challenge as it was unknown whether getting a regret smaller than $\Omega(T^{\frac{2}{3}})$ was possible. 
    We answer positively this question, as a natural extension of UCB exhibits a $\Ocal(\sqrt{T\log(T)})$ minimax regret. 
    More importantly, we introduce Cautious Greedy, a centralized algorithm that yields constant instance-dependent regret if the optimal policy assigns at least one player on each arm (a situation that is proved to occur when arm means are close enough). Otherwise, its regret increases as the sum of $\log(T)$ over some sub-optimality gaps. We provide lower bounds showing that Cautious Greedy is optimal in the data-dependent terms.
   Therefore, we set up a strong baseline for asynchronous multiplayer bandits and suggest that learning the optimal policy in this problem might be easier than thought, at least with centralization.
\end{abstract}

\section{Introduction}

\label{sec:introduction}
In the classical multi-armed bandits (MAB) problem, a single player sequentially pulls  arms $k_t \in \{ 1, \dots, K \} \defeq [K]$, and receives a reward $X_{k_t}$ sampled from some unknown sub-Gaussian distribution of mean~$\mu_{k_t}$.
This process undergoes repetition for a total of $T$ rounds and the performance of the sampling policy is measured by its regret, the difference between the total expected reward obtained by choosing the best arm $k^*$ at each round 
and the total expected reward of the player's actual choices.
This setting has been extensively studied \citep[see][for a survey]{lattimoreBanditAlgorithms2020}.
A fundamental component of MAB is the exploration and exploitation trade-off, as a good policy should balance between both. Exploration involves trying out different arms to gather information, while exploitation uses the acquired knowledge to favor arms more likely to be the best. Optimal policies are known to have a regret scaling as $\Ocal(\sum_{k\neq k^*}\frac{\log(T)}{\mu_{k^*}-\mu_k})$~\citep{auer2002finite}.

Classical applications of MAB include clinical trials, recommendation systems or ad placements. For many applications, the MAB framework however does not fit the problem at hand. Consider for instance cognitive radios \cite{laiMediumAccessCognitive2008, anandkumarDistributedAlgorithmsLearning2011, mitola1999cognitive,jouini2010upper} where arms correspond to communication channels available to radio devices. What differs from standard MAB is that if two radios choose the same communication channel, they interfere. This example motivates the multiplayer multi-armed bandits ($\MMAB$) setting introduced in \cite{liu2010distributed}. In $\MMAB$, $M$ players simultaneously pull arms. When a player pulls arm $k$, it receives the reward $\eta_k X_k$ where $\eta_k = 0$ if two or more players collide, meaning they pull the same arm $k$, and $\eta_k=1$ if a single player pulls $k$. In the centralized setting with $M \leq K$, this setting is equivalent to bandits with multiple plays \citep{komiyamaOptimalRegretAnalysis, anantharamAsymptoticallyEfficientAllocation1987, chenCombinatorialMultiarmedBandit2013, gopalanThompsonSamplingComplex2013}, where a central entity decides on the behalf of agents and trivially avoids collisions. Optimal algorithms are then known to yield an asymptotic regret $\sum_{k =1}^{K-M}\frac{\log(T)}{\mu_{(K - M + 1)} - \mu_{(k)}}$ \citep{komiyamaOptimalRegretAnalysis}, where $\mu_{(i)}$ is the $i$-th smallest mean reward.

Motivated by Internet of Things networks, we focus on the asynchronous multiplayer multi-armed setting (AMMAB) where each round is decomposed into three successive steps \citep[see][]{dakdoukMassiveMultiplayerMultiarmed2022, bonnefoiMultiArmedBanditLearning2017}. First, all players decide which arm they would like to play. Second, the environment activates independently player $i$ with probability $p_i$. In the third step, activated players pull the arm they chose in the first step. In this model, players  correspond to communicating devices, arms to available channels and $p_i$ is the activation probability of the communicating device $i$. 

We focus on the centralized setting (equivalently, agents' communication is free). Although the decentralized setting might be more fitted to communication network applications, the centralized setting is already challenging enough to warrant a study (furthermore, centralized algorithm performances are benchmarks for decentralized ones, hence it is crucial to investigate the former). Possible extensions, including the decentralized case, are discussed in \Cref{sec:conclusion}.

\paragraph{Notations.} Vectors are denoted in bold. If $\ub \in \RR^n$, $u_i$ is the $i$-th coordinate of $\ub$ while $u_{(i)}$ the $i$-th smallest coordinate of $u$ and $\support(\ub) = \{ i \in [n], u_i \neq 0 \}$. We denote for $\ub, \vb \in \RR^n$, $\langle \ub, \vb \rangle = \sum_{i=1}^n u_i v_i$, $\|\ub\|_\infty = \max_{i \in [n]} |u_i|$, $\|\ub\|_1 = \sum_{i \in [n]} |u_i|$. For a function $f:\RR \rightarrow \RR$ and $\ub \in \RR^n$, $f(\ub) \in \RR^n$ is defined by $f(\ub)_i = f(u_i)$. 
Lastly, $\overline{E}$ denotes the complementary event of $E$.

\paragraph{Setting and assumptions.}
For simplicity, we follow \citet{bonnefoiMultiArmedBanditLearning2017} and assume that the probability of being active is the same for all players: $p_i = p$, for all $i \in [M]$.
This makes players exchangeable and allows for a simplified description of the AMMAB setting. 
At each round $t$, a central entity chooses a (possibly random) assignment $\Mb(t) = (M_1(t), \dots, M_K(t))$ where $M_k(t)$ is the number of players assigned to arm $k$ at round~$t$; the environment then activates each player with probability\footnote{The central entity does not know beforehand which players will be active, making collisions unavoidable.}~$p$ and active players pull the arm to which they are assigned, each receiving reward $\eta_k(M_k(t)) X_k$ where $\eta_k = \ind{\text{exactly one player is active on arm $k$}}$. 
Active players communicate to the central entity their earned reward $X_k \eta_k$. Additionally, the central entity observes the collision events $(\eta_k)_{k \in [K]}$, and the parameters $M$, $K$ and $p$ are assumed to be known beforehand.

At any time $t$, the assignment $\Mb(t)$ satisfies the budget constraint $\sum_{k=1}^K M_k(t)= M$ and we assume:
\begin{assumption}
$M \geq K$ and for all $k \in [K]$ and at all stages $M_k(t) \leq \frac{-1}{\log(1 - p)}\simeq \frac{1}{p}$.
\end{assumption}
 The second condition is not restrictive, as assigning more than $\frac{-1}{\log(1 - p)}$ players on the same arm only decreases the obtained reward and amount of information on that arm. A better policy would then have some players not play at all instead, or equivalently assign players to a dummy arm whose reward is known to be $0$.
The set of valid assignments is thus denoted by $$\mathcal{M} = \{ \Mb \in [M]^K \mid \sum_{k=1}^K M_k= M, M_k \leq \frac{-1}{\log(1 -p)} \}.$$
The goal is  to minimize the expected regret defined by:
\begin{equation}
\EE[R] = \sum_{t=1}^T \sum_{k=1}^K  \EE[\eta_k(M^*_k) X_k] - \EE[\eta_k(M_k(t)) X_k]
\label{eq:regret}
\end{equation}
where $\Mb^* = (M_1^*, \dots, M^*_K)$ is the optimal assignment:
\begin{equation}
\Mb^* = \argmax_{\Mb \in \mathcal{M}}\EE\left[ \sum_{k=1}^K \eta_k(M_k) X_k\right].
\label{eq:maxproblem}
\end{equation}
\citet{bonnefoiMultiArmedBanditLearning2017} designed an algorithm solving \Cref{eq:maxproblem} with known $\mu_k$, 
%
and \citet{dakdoukMassiveMultiplayerMultiarmed2022} later proposed a simpler sequential algorithm. 
In combination with some explore-then-commit policy, it yields a regret scaling in $\Ocal(T^{\frac23})$. Additionally, \cite{dakdoukMassiveMultiplayerMultiarmed2022} show there is no random assignment yielding a strictly larger expected reward than the deterministic optimal assignment $\Mb^*$.

\paragraph{Contributions.}
We prove that an adapted version of UCB exhibits a regret in $\Ocal(\sqrt{TK \log(T) \min(Mp, K)})$ where $\Ocal(\cdot)$ hides universal constant factors. This result largely outperforms the $\Ocal(T^{\frac23})$ regret bound of the ETC algorithm by~\citet{dakdoukMassiveMultiplayerMultiarmed2022}.
Contrary to the lower bound by~\citet[][Theorem 2]{wang2017improving}, the $1/p$ term does not appear in this bound, as the two settings are slightly different: although rewards are observed with probability~$p$, rewards are also scaled by $p$ in our case (thus canceling out the terms in $p$). More surprisingly, our main contribution shows that achieving a constant regret (in $T$) is sometimes possible with an algorithm called Cautious Greedy. The analysis of centralized UCB is thus postponed to \Cref{app:ucb} and the main text solely focuses on Cautious Greedy.
In essence, it is a standard greedy algorithm that estimates $\mu_k$ via empirical means, but it is cautious as it avoids assigning zero players to an arm unless, with high confidence, assigning no players to it is optimal. More precisely, Cautious Greedy maintains a lower bound $\nu$ of the number of arms that should be assigned zero players and stops assigning players to the $\nu$ worst arms when confident enough.

The regret of Cautious Greedy depends on several data-dependent quantities  defined in \Cref{sec:analysis}: \begin{itemize}
    \item 
$\nu^*$ the number of arms that are assigned zero players in the optimal assignment,  
\item $\Delta_{(j)} = \mu_{(\nu^*+1)} - \mu_{(j)}$, \item $\Mb_{\nu}^*$ the optimal assignment when $\nu$ arms are assigned zero players and 
\item $r$ the infinity norm of the minimum perturbation of the true arm means $\mub$ that would modify the sequence $(\Mb^*_{\nu})_{\nu=1}^{\nu^{*}}$.\end{itemize}
\Cref{prop:greedy:up} together with \Cref{lemma:third term} show that the regret of Cautious Greedy is upper bounded by $\Ocal\left(\frac1{r}+ \sum_{j \leq \nu^*} \frac{\log(T)}{\Delta_{(j)}}\right)$, where $\Ocal$ hides terms depending on $K, p$ and $M$.

In particular, Cautious Greedy achieves constant regret if $\nu^*=0$, i.e., when each arm is assigned at least one player by the optimal policy.  As shown by the lower bound in \Cref{greedy:lb:constant}, under mild conditions, the dependency in $\frac1{r}$ cannot be improved. In \Cref{lemma:greedy:necessary condition elimination}, we give a sufficient condition on the dispersion of arm means to get $\nu^{*}=0$.
In general, Cautious Greedy suffers an additional dependency in $\sum_{j \leq \nu^*} \frac{\log(T)}{\Delta_{(j)}}$. This dependency also appears in bandits with multiple plays~\citep{komiyamaOptimalRegretAnalysis} and as shown by \Cref{greedy:lb:gap} cannot be removed. This makes Cautious Greedy optimal with respect to $T$ and with respect to the data-dependent quantities $r$ and $\Delta_{(j)}$. 

The main difficulty of the problem comes from the fact that $\nu^{*}$ is unknown. A classical Greedy algorithm yields a linear regret when $\nu^{*} > 0$, while a traditional bandits algorithm never reaches constant regret when $\nu^{*} = 0$. On the other hand, Cautious Greedy performs optimally in both cases.

\Cref{sec:experiments} benchmarks Cautious Greedy against our UCB algorithm and the centralized ETC algorithm of~\citet{dakdoukMassiveMultiplayerMultiarmed2022} on synthetic data and show that Cautious Greedy and UCB perform both significantly better than ETC. Cautious Greedy outperforms UCB when no arms should be assigned zero players while UCB tends to be better when at least one arm should be assigned zero players.

\section{Related work}

\paragraph{Centralized setting: multiplay, combinatorial and structured bandits.} As already highlighted, when $M \leq K$ and $p=1$, AMMAB is equivalent to bandits with multiple plays. A lower bound in $\sum_{j=1}^{\nu^*} \frac{\log(T)}{\Delta_{(j)}}$ where $\nu^* = K - M$ is shown in \cite{anantharamAsymptoticallyEfficientAllocation1987}. This lower bound is reached by a Thompson sampling-based algorithm~\citep{komiyamaOptimalRegretAnalysis}. Bandits with multiple plays are an instance of combinatorial bandits~\citep{gai2012combinatorial, chen2013combinatorial, kveton2015tight, combes2015combinatorial, wang2018thompson, perrault2020statistical} where an agent chooses an action $\ab \in \Scal$ and receives reward $r(\mub, \ab)$. When $M \leq K$, AMMAB is an instance of combinatorial bandits with semi-bandit feedback and probabilistically triggered arms (meaning that chosen arms are triggered with some probability) \citep{wang2017improving, chen2016combinatorial} with the difference that in these works, rewards are scaled by $1/p$. More generally, AMMAB can be viewed as combinatorial bandits or structured bandits \cite{combes2017minimal} with semi-bandit feedback and $KM$ possible actions. 
None of these works yet allow to reach constant regret when $\nu^*=0$.

\vspace{-0.2cm}
\paragraph{Decentralized multiplayer bandits.} 
In decentralized multiplayer bandits,  players aim at speeding up the collective learning of the arm rewards, while avoiding collisions. Motivated by cognitive radio networks, the decentralized problem of multiplayer bandits recently received a lot of attention~\citep[we refer to][for a review]{boursier2022survey}, sometimes assuming a pre-agreement on the ranks of the players~\citep{anandkumar2010opportunistic,liu2010distributed} or using few collisions to communicate information between players \citep{avner2014concurrent, rosenski2016multi, besson2018multi}. However, \citet{bistritz2018distributed,boursierSICMMABSynchronisationInvolves2019, wangOptimalAlgorithmsMultiplayer2020} enforce collisions to send a significant number of bits between the players, allowing to reach optimal centralized performance. This idea is also used in many extensions of MMAB~\citep{mehrabian2020practical,shi2020decentralized, huang2021towards,boursier2020selfish,shi2021heterogeneous}. This communication through collision trick yet highly depends on the synchronicity of the players and becomes costly with a lot of players. In AMMAB, the players are asynchronous ($p<1$) and numerous ($M\geq K$), making both drawbacks significant. This work thus proposes a centralized asynchronous algorithm, leaving open for future work a possible decentralized adaptation (see \Cref{sec:conclusion} for a discussion). 

 In the multi-agent bandit problem considered by \citet{szorenyi2013gossip,landgren2016distributed,martinez2019decentralized}, no collision happens when several players pull the same arm. The problem is thus different in nature: the main objective of multi-agent bandits is to speed up learning using decentralized communication protocols (e.g. gossip), without consideration of collision.  
%
\vspace{-0.2cm}

\paragraph{Full information.} When each arm is assigned at least one player, it provides information with a strictly positive probability at each time step. Therefore in this regime, the central entity is almost in full information feedback, where information about all arms is received at every round. Bandits with expert advice are  examples of problems with full information feedback. The go-to algorithm in the adversarial setting is (variants of) exponential weights or Hedge \citep{mourtadaOptimalityHedgeAlgorithm2019}. However, in the stochastic setting, a constant regret is achieved by Greedy (aka Follow The Leader) which plays according to the empirical mean estimate of the rewards~\citep{degenneAnytimeOptimalAlgorithms2016}. \cite{huangFollowingLeaderFast2017} show Greedy achieves constant regret in a more structured setting. 
\vspace{-0.2cm}

\paragraph{Resource allocation.}
Our problem is also a particular instance of sequential resource allocation with concave utilities~\citep{lattimore2015linear,fontaine2020adaptive,zuo2021combinatorial}. Although general resource allocation algorithms could be used in our setting, much better solutions can be obtained by leveraging the very specific structure of the utilities. The utility functions are indeed exactly known up to the multiplicative factor $\mu_k$.
%
\vspace{-0.2cm}

\paragraph{Asynchronous multiplayer bandits.}
AMMAB was introduced by \cite{bonnefoiMultiArmedBanditLearning2017} in the context of cognitive radios. In~\cite{dakdoukMassiveMultiplayerMultiarmed2022}, players have heterogeneous activation probabilities. In addition, the authors keep track of the amount of communication between agents. They propose an explore and commit algorithm that reaches a sub-optimal $\Ocal(T^{\frac23})$ regret. In contrast, we show that under favorable conditions, a constant regret can be reached. Extending our results to the heterogeneous setting of~\citet{dakdoukMassiveMultiplayerMultiarmed2022} remains open for future work. 
Quite interestingly in AMMAB, the expected individual reward decreases as more players are assigned to the same arm. This relates the AMMAB model to more advanced collision models for MMAB, where a collision only decreases the reward instead of yielding a $0$ reward~\citep{tekin2012online, bande2019multi, magesh2019multi, boyarski2021distributed}. 
AMMAB is also related to the problem of online queuing systems~\citep{gaitonde2020stability,sentenac2021decentralized}, where packets arrive in a queue (player) with random rates. This setting yet differs from AMMAB, as players are active as long as they hold packets.

\section{Cautious Greedy, an efficient centralized algorithm for AMMAB}
\label{sec:cautious greedy}

Let us first define the function $g(x) = x p (1 - p)^{x-1}$, so that the regret in \Cref{eq:regret}  rewrites as  
\begin{equation}
\EE[R] = \sum_{t=1}^T \EE\big[ \big\langle \mub, g(\Mb^*) - g(\Mb(t))\big\rangle\big]
\label{eq:regret2}
\end{equation}
where $\Mb^*= \argmax_{\Mb \in  \mathcal{M}} \langle \mub, g(\Mb) \rangle$ is a rewritting of \Cref{eq:maxproblem}.

\subsection{Description}
\label{greedy:desc}
Cautious Greedy is based on a standard greedy strategy that plays the best policy according to the estimated mean rewards. Cautious Greedy therefore computes mean rewards estimates~$\hat{\mub}(t)$:
\begin{align}
\hat{\mu}_k(t) = \frac{\sum_{\rho=1}^t \eta_k^\rho(\Mb(\rho)) X_k^\rho}{T_k(t)}
\label{eq:mu_hat}
\end{align}
where $T_k(t) = \sum_{\rho=1}^t \eta_k^\rho(\Mb(\rho))$ is the number of samples gathered from arm $k$ at time $t$ or equivalently, the number of times that arm $k$ has been played by exactly one player up to time $t$. By convention, we set $\hat{\mu}_k(t) = 1$ if $T_k(t) = 0$.

A Greedy algorithm would then choose the assignment $\Mb(t) = \Mb^{\hat{\mu}(t)}_{\Mcal}$ where
\begin{equation}
\label{eq:mb nu:notation}
    \Mb^{\hat{\mu}}_{\Mcal} = \argmax_{\Mb \in  \mathcal{M}} \langle \hat{\mub}, g(\Mb) \rangle.
\end{equation}
Such a simple strategy would quickly stop exploring, at the risk of committing to a suboptimal policy. In order to maintain some level of exploration, a natural idea is to impose at least one player per arm.
However, in some settings, the optimal solution might assign no players to some arms. The challenging task of Cautious Greedy is then to identify which arms should be assigned zero players. We call such identified arms \emph{removed} while \emph{active arms} are those not removed yet.
Cautious Greedy can put a set of arms $\mathcal{S}$ \emph{under pressure}, meaning that these arms are temporally allowed to be assigned no player. Arms that are assigned at least one player are said to be \emph{played} and note that it is possible that an arm under pressure is played.
Formally, the constraints that apply to $\Mb$ in the assignment problem will be described by sets of the form:
\[
 \mathcal{M}_{\mathcal{K}} = \{ \Mb \in \mathcal{M},  \forall k \in \mathcal{K}, M_k \geq 1 \}
 \]
where $\mathcal{K} \subset [K]$. $\mathcal{M}_{\mathcal{K}}$ is the set of assignments that put under pressure $[K]\setminus\mathcal{K}$.
In order to identify the arms to remove, Cautious Greedy maintains confidence bounds on the mean of each arm.
The upper and lower bounds are given respectively by 
\begin{gather}
    \hat{\mub}^H(t) = \min(\hat{\mub}(t) + \zetab(t), 1) \label{eq:muh}\quad \text{and}\quad\hat{\mub}^L(t) = \max(\hat{\mub}(t) - \zetab(t), 0)\\
\text{where for all }k \in [K], \qquad   \zeta_k(t) = \sqrt{\frac{\log(2 T^2 K^2)}{2 T_k(t)}}.
  \label{eq:zeta}
 \end{gather}
These bounds are used to eliminate sub-optimal arms.
This could suggest a strategy that plays all active arms at each round until enough information is gathered to remove an arm. However, such a strategy yields high regret in the case where two arms that should be eliminated are very close to each other. Therefore, the elimination of several arms at once is allowed.
This is done in Cautious Greedy by computing an estimate $\nu$ of the number of arms to remove, which is a lower bound of $\nu^* = \left|\{ k, M_k^* = 0 \}\right|$ and can be used to eliminate several arms at once without ordering them first.
We therefore introduce $\Mcal_{\nu}$ the set of assignments where $\nu$ arms are under pressure:
\[
 \mathcal{M}_{\nu} = \{ \Mb \in \mathcal{M},  | \support(\Mb)| \geq K - \nu \}.
\]
The number of arms to remove $\nu$ is then increased when $\langle \hat{\mub}^L, g(\Mb^{\hat{\mub}^L}_{\Mcal}) \rangle > \langle \hat{\mub}^H, g(\Mb^{\hat{\mub}^H}_{\Mcal_{\nu}}) \rangle$, i.e. when a larger reward is guaranteed by removing more than $\nu$ arms.
Cautious Greedy then uses $\nu$ to build a set $\mathcal{A}$ of \emph{accepted arms} which are arms that are not likely to be among the $\nu$ worst arms. Cautious Greedy then puts under pressure a subset of arms among the arms that are not accepted yet. The set of arms put under pressure rotates in a \emph{round-robin} fashion. 
This mechanism ensures that all active arms are regularly played. After the round-robin rotation is completed, Cautious Greedy reevaluates $\nu$ and updates the sets of accepted arms and active arms. As $\nu$ increases, an arm can be removed from the set of accepted arms. However as $\nu$ never decreases, a removed arm is removed forever.
The exact procedure is described in \Cref{algo:greedy} below.
\begin{algorithm}
	\caption{Cautious Greedy}
  \label{algo:greedy}
	\begin{algorithmic}[1]
		\STATE {\bfseries Input :} $M$ (number of players), $p$ (probability that a player is active), $T$ (horizon)
    \STATE $\nu = 0$ \COMMENT{Estimate of $\nu^*$}
    \STATE Initialize the set of active arms $\mathcal{K} = [K]$; the set of accepted arms $\mathcal{A} = \emptyset $\\ Initialize the set of arms under pressure $\mathcal{U} = \emptyset $; the round-robin counter $n=0$
    \FOR{$t=1,\ldots,T$}
    \STATE Play $\Mb^{\hat{\mub}}_{\mathcal{M}_{\Ecal}}$ as defined in~\eqref{eq:mb nu:notation} where $\Ecal = \Kcal \setminus \mathcal{U}$
    \STATE Rotate $\Ucal$ in a round robin fashion over $\mathcal{K} \setminus \mathcal{A}$ (See \Cref{sec:RR} for details)
    \STATE Update $\hat{\mub}$ according to~\eqref{eq:mu_hat}; $n = n+1$
    \IF[end of round robin]{$n= |\mathcal{K} \setminus \mathcal{A}|$} \label{phase:condition}
    \STATE $n = 0$
    and compute $\Mb^{\hat{\mub}^L}_{\Mcal}$ and $\Mb^{\hat{\mub}^L}_{\Mcal_{\nu}}$ following \eqref{eq:mb nu:notation}
    \WHILE{$ \langle \hat{\mub}^L, g(\Mb^{\hat{\mub}^L}_{\Mcal}) \rangle > \langle \hat{\mub}^H, g(\Mb^{\hat{\mub}^H}_{\Mcal_{\nu}}) \rangle$}
    \label{algo:greedy:while}
    \STATE $\nu = \nu + 1$
    \ENDWHILE
    \STATE Update $\mathcal{A} = \{ k \in [K], \hat{\mu}_{(\nu)}^H < \hat{\mu}_k^L \}$ \hspace{10pt} and \hspace{10pt} $\mathcal{K} =[K] \setminus \{ k \in [K], \hat{\mu}_k^H < \hat{\mu}^L_{(\nu + 1)} \}$
    \STATE Let $\mathcal{U}$ be $\nu - |[K] \setminus \mathcal{K}|$ elements from $\mathcal{K} \setminus \mathcal{A}$
    \ENDIF
    \ENDFOR
	\end{algorithmic}
\end{algorithm}

\subsection{Regret bound}
\label{sec:analysis}
The main result of this section is an upper bound on the expected regret of Cautious Greedy. This bound depends on several data-dependent quantities that we now define precisely. $\Delta^{(\nu^*)}$ is the minimum simple regret achieved by an allocation removing exactly $\nu^*-1$ arms, while the number of arms removed by the optimal assignment is equal to $\nu^*$. Denoting $\Mb^*_\nu = \Mb^{\mub}_{\Mcal_{\nu}}$, $\Delta^{(\nu^*)}$ is defined as $\Delta^{(\nu^*)} = \langle \mub, g(\Mb^*) - g(\Mb^*_{\nu^*-1}) \rangle$. By convention, we set $\Delta^{(\nu^*)} = \infty$ if $\nu^*=0$. 
$\Delta_{(j)} = \mu_{(\nu^* + 1)} - \mu_{(j)}$ is the difference between the reward of the worst arm not eliminated in the optimal assignment and the reward of the $j$-th worst arm. Lastly, $r$ is the norm of the minimum perturbation of $\mub$ causing $\Mb^*_\nu$ to change for some value of $\nu$. More precisely, define $r_\nu = \min_{\hat{\mu}, \Mb^{\hat{\mub}}_{\Mcal_{\nu}} \neq \Mb^*_\nu} \|\hat{\mub} - \mub \|_\infty$, then $r = \min_{\nu \in [\nu^*]} r_{\nu}$. \Cref{prop:greedy:up} shows that the expected regret of Cautious Greedy is upper bounded by $\Ocal(\frac{\nu^* + 1}{r} + \nu^* \frac{\log(T)}{\Delta^{(\nu^*)}} + \sum_{j \leq \nu^*} \frac{\log(T)}{\Delta_{(j)}})$ where $\Ocal$ hides quantities independent of the data and $T$.

\begin{proposition}[Upper bound on the regret Cautious Greedy]
\label{prop:greedy:up}
The expected regret of Cautious Greedy satisfies
\begin{align*}
  \EE[R] \leq &\ 20 \frac{KM(\nu^* + 1)}{r} + \sum_{\nu=1}^{\nu^*} \frac{120 \log(2 T^2 K^2)}{\Delta_{(\nu)}} \\ &+ \bigg[72 M \min(Mp, K) (\nu^* + 1) + 120\bigg]\frac{\log(2K^2 T^2)}{\Delta^{(\nu^*)}} .
\end{align*}
\end{proposition}
The first term is reminiscent of the regret induced by Greedy with full information. The second one comes from the sample complexity of finding the $\nu^*$ worst arms. The third one is finally due to the sample complexity of  detecting that the optimal policy eliminates $\nu^*$ arms.
Interestingly, the two last terms are null when $\nu^*=0$, which corresponds to situations where the optimal policy assigns at least one player on every arm. This makes the regret of Cautious Greedy constant  in such situations, which happens when arm rewards have a similar order of magnitude (see \Cref{lemma:greedy:necessary condition elimination}).

At first sight, it seems like the third term in \Cref{prop:greedy:up} could be arbitrarily larger than the second term. Fortunately, this is untrue as shown in the following Lemma:
\begin{lemma}
	\label{lemma:third term}
	$
\Delta^{(\nu^{*})} \geq (g(M_{(\nu^*+1)}^* + 1)  - g(M_{(\nu^*+1)}^*)) \Delta_{{(\nu^{*})}}.$
\end{lemma}

Together with \Cref{prop:greedy:up}, \Cref{lemma:third term} shows that the regret of cautious Greedy is upper bounded by $\Ocal(\frac1{r} + \sum_{\nu=1}^{\nu^{*}} \frac{\log(T)}{\Delta_{(\nu)}})$ where $\Ocal$ hides terms in $M, p, K$.
The remainder of this section sketches the proof of \Cref{prop:greedy:up}. The precise statement of lemmas and their proofs are deferred to \Cref{app:proof}. 
\begin{proof}[Proof sketch of \Cref{prop:greedy:up}]
    Using classical concentration bounds (\Cref{lemma:concentration:mu}), we can assume that $\mub^H$ and $\mub^L$ (defined in \Cref{eq:muh}) verify $\mub^H \geq \mub \geq \mub^L$ without affecting the regret bound. 
    
    Consequently, \Cref{algo:greedy} ensures that $\nu$ is only increased if $\nu < \nu^*$ (\Cref{lemma:mu:ub}) and the update of the set of active arms ensures that optimal arms are never eliminated (\Cref{lemma:no good arms eliminated}).

    We then focus on bounding the number of times each arm is played. The round-robin procedure ensures that all active arms are assigned at least one player regularly, as proven by \Cref{lemma:tau}. However, because of collisions, assigning at least one player to an arm does not guarantee an observation. \Cref{lemma:q} makes this relation explicit.
    Now knowing how many observations are gathered on each arm, we can focus on upper bounding the regret.

    Denote $\Mb^*_\nu = \Mb^{\mub}_{\Mcal_{\nu}}$ the optimal assignment of players when at most $\nu$ arms can be assigned zero players. For $\mathcal{E}(t) \subset [K]$, call $\Mb^*_{\mathcal{E}(t)}= \Mb^{\mub}_{\Mcal_{\mathcal{E}(t)}}$ the optimal assignment of players when only arms not in $\mathcal{E}(t)$ can be assigned zero players. We can write the cost of the chosen assignment at time $t$ $\Mb(t)$ as the sum of three terms:
    \[
    \underbrace{\langle \mub,  g(\Mb^*) - g(\Mb^*_{\nu}) \rangle}_{(i)} + \underbrace{\langle \mu,  g(\Mb^*_{\nu}) - g(\Mb^*_{\mathcal{E}(t)})\rangle}_{(ii)} \\
       + \underbrace{\langle \mu, g(\Mb^*_{\mathcal{E}(t)}) - g(\Mb(t)) \rangle}_{(iii)} 
    \]

These three terms measure a different aspect of the regret:  $(i)$ measures the error due to $\nu$ the number of arms under pressure being different from $\nu^*$ the optimal number of players to eliminate; $(ii)$ measures the error due to $\mathcal{E}(t)$ being different from $\support(\Mb^*_\nu)$, the optimal set of arms that must be assigned at least one player by $\Mb^*_\nu$; $(iii)$ measures the error due to $\Mb(t)$ being different from $\Mb^*_{\mathcal{E}(t)}$, the optimal assignment of players  among possible assignments in $\mathcal{M}_{\mathcal{E}(t)}$.

Let us start with (i). As the number of samples seen increases, $\nu$ increases to get closer to $\nu^*$. \Cref{lemma:greedy:nu} bounds the number of samples seen before the algorithm increases $\nu$, which leads to an upper bound on the total regret due to this term shown in \Cref{lemma:greedy:bound:nu}.

Regarding (ii), for a given $\nu$, two things may prevent a sub-optimal choice of arms $\mathcal{E}$ on which at least one player must be assigned. Either an arm in $\mathcal{E}$ is eliminated or an arm in $[K] \setminus \Ecal$ is accepted. \Cref{lemma:greedy:arm elimination} provides a lower bound on the number of samples seen before a sub-optimal arm is eliminated while \Cref{lemma:greedy:arm acceptation} provides a lower bound on the number of samples seen before an optimal arm is accepted.
The two previous lemmas allow to quantify when arms are accepted or rejected. We then compute the cost of a sub-optimal choice of arms $\mathcal{E}$ in \Cref{lemma:greedy:cost:e} and combine these three lemmas to bound the total regret due to this term in \Cref{lemma:greedy:bound:e}.

Lastly, the third term $(iii)$ measures the mismatch between the chosen assignment $\Mb(t)$ and the best possible assignment with the same support. Crucially there is no support mismatch and therefore we are in a setting close to the full information setting which allows us to bound the regret due to these terms by a quantity independent of the horizon $T$ (see \Cref{lemma:rm bound}).

Adding the upper bounds due to the terms $(i)$, $(ii)$ and $(iii)$ and reorganizing concludes the proof.
\end{proof}

\vspace{-0.2cm}

\section{Lower bound}
The next lemma lower bounds the best possible constant term in the regret. Under mild conditions, it shows there exists a choice of rewards $\mub$ such that any algorithm has a regret scaling in $\Ocal(\frac1{r})$.

\begin{lemma}[Lower bound for $\nu^*=0$]
  \label{greedy:lb:constant}
  Consider $K=2$ arms and $M = 2N + 1$ players for some $N \in \NN^*$ and assume $p \leq \frac1{M+1}$, $r_0 < \frac{p}{12}$, $T \geq \frac{1}{16 g(M) r_0^2}$. For any algorithm $A$, there exists a choice of rewards $\mub$ such that $r(\mub) = r_0$ and  
  \[
  \EE[R_A] \geq \frac{1}{256 M r_0}.
  \]
\end{lemma}
\begin{proof}[Proof sketch (see proof in \Cref{lemma:greedy:lb:constant}) ]
We take parameters $\mub_1$ and $\mub_2$ such that $r = \frac{\Delta}{2}$ and the optimal solution is $\Mb^* = (N, N+1)$ if $\mub = \mub_1$ and $\Mb^* = (N+1, N)$ if $\mub = \mub_2$. Moreover, we choose them so that the top two solutions are always $(N, N+1)$ and $(N+1, N)$.
First, we \textit{augment} $A$ so that each arm yields a sample $X_k$ with probability $g(M)$ instead of $g(M_k(t))$; moreover $A$ is forced to chose at each step between $\Mb(t) = (N, N+1)$ or $\Mb(t) = (N+1, N)$ (these two modifications only improve $A$).  Following the proof of Theorem~3 in \cite{wang2017improving}, we recast this setting as a 2-armed bandit problem where arm $k$ has reward $1$ with probability $g(M) \mu_k$, $0$ with probability $g(M)(1 - \mu_k)$ and $X_k = \perp$ with probability $1 - g(M)$. The rest of the proof follows closely the proof of Proposition~4 in~\cite{mourtadaOptimalityHedgeAlgorithm2019} and yields the lower bound $\EE[R_A] \geq \frac{\Delta (g(M+1) - g(M))}{2} \frac{T}{4} \exp(-4 T g(M) \Delta^2)$. As the regret increases with $T$, taking $T = \floor{\frac{1}{4 g(M) \Delta^2}}$ concludes. 
\end{proof}

The upper bound of Cautious Greedy when $\nu^* = 0$ is given by $ \EE[R] \leq \frac{20 K M}{r}$, i.e., the dependency in $r$ cannot be improved.
Next, we investigate the case $\nu^* > 0$ and show a lower bound inspired by the classical results of \cite{lai1985asymptotically}. Let us first introduce the notion of a consistent algorithm.
  Let $T_i$ be the number of times with at least one player on the $i$-th worst arm.
  An algorithm is consistent if $\forall \alpha > 0$, $\forall j > \nu^* \in \EE[T-T_j] = \Ocal(T^\alpha)$ and $\forall j \leq \nu^*, \EE[T_j] = \Ocal(T^\alpha)$.

\begin{lemma}[Lower bound for $\nu^* > 0$]
  \label{greedy:lb:gap}
 For any integers $M \geq 5, \nu^* > 0, p \leq \frac1{M + 1}$, any gaps $\Delta_{(1)}, \dots, \Delta_{(\nu^{*})} \leq \frac{p}{8 (M-4)}$,  and for any consistent algorithm $A$, there exists a set of parameters $(\mu_1,\ldots, \mu_{\nu^*+2})$ such that $\mu_{(\nu^{*}+1)} - \mu_{(\nu)}=  \Delta_{(\nu)}$ for all $\nu \in [\nu^*]$ and the regret of $A$ satisfies \[
 \liminf_{T \to \infty} \frac{\EE R_A}{\log(T)} \geq \sum_{\nu=1}^{\nu^*} \frac{c}{\Delta_{(\nu)}}
\]
for some universal constant $c>0$.
\end{lemma}
\begin{proof}[Proof sketch (see proof in \Cref{lemma:greedy:lb:constant2})]
%
Assume for the sketch of proof that $\nu^*=1$ and  that there are $3$ arms. We are considering two alternative mean parameters $(\mu_0,\mu_1,\mu_1+\Delta)$ and $(\mu_0,\mu_1,\mu_1-\Delta)$ chosen so that the optimal allocation is either $(M-1,1,0)$ or $(M-1,0,1)$. Moreover, we choose $\mu_0$ and $\mu_1$ such that in both worlds, the top two allocations are always the aforementioned ones. 
This might give the impression that there exists a trivial reduction to some standard 2-arm bandits (where those arms are the tentative two optimal allocations). A consistent algorithm would indeed need $N^*:=\Omega(\frac{\log(T)}{\Delta^2})$ samples of sub-optimal arms to distinguish between the two worlds. In particular, with the second set of parameters, this  requires putting one player on the third arm $N^*/p$ times (in expectation), each one incurring a cost of $p\Delta$. This would give the result for $\nu^*=1$ and this technique can be immediately generalized to $\nu^*>1$.

It is however not that simple, as putting more players on some (suboptimal) arm gives faster feedback, yet at a higher cost. We yet show that the best trade-off (in feedback received vs. suboptimality cost) for an algorithm to distinguish between the two worlds is indeed to allocate a single player on arm $2$ or $3$. The aforementioned intuition is thus actually correct but requires a cautious argument.
\end{proof}

\Cref{greedy:lb:gap} shows that the dependency in $\sum_{j \leq \nu^*} \frac{\log(T)}{\Delta_{(j)}}$ in the upper bound of \Cref{prop:greedy:up} cannot be improved.
The results in \Cref{greedy:lb:constant} and \Cref{greedy:lb:gap} show that Cautious Greedy is optimal with respect to its dependency in $r$, $\Delta_{(\nu)}$ and $T$. However, we do not claim that the dependency in $M, p$, or $K$ is optimal. Improving the dependency with respect to these parameters would be an interesting although challenging direction for future work.
\vspace{-0.2cm}

\section{Experiments}
\label{sec:experiments}
Our experiments compare the expected regret of Cautious Greedy (\Cref{algo:greedy}), UCB (\Cref{algo:ucb}), and ETC \citep[Algorithm 8]{dakdoukMassiveMultiplayerMultiarmed2022}. In all these algorithms, maximization problems of the form $\max_{\Mb \in \Mcal} \langle g(\Mb), \vb \rangle$ are solved using the sequential algorithm of \citet[Algorithm 5]{dakdoukMassiveMultiplayerMultiarmed2022}. In Cautious Greedy, the sequential algorithm is also adapted to solve $\max_{\Mb \in \Mcal_{\Ecal}} \langle g(\Mb), \vb \rangle$ for some set $\Ecal \subset [K]$. This is done by first assigning one player to each arm in $\Ecal$ and then running the sequential algorithm for the rest of the players. The optimality of this approach is detailed in \Cref{app:optimal:max}.

At each step $t \in [T]$, algorithms compute a player assignment $\Mb(t) \in \Mcal$ based on the rewards they have seen so far. We record $\sum_{\tau=1}^t \langle \mub, g(\Mb^*) - g(\Mb(\tau)) \rangle$ which we call cumulative regret (instead of pseudo regret). Each experiment is run $50$ times, we plot the mean value of the cumulative regret as a function of $t \in [T]$. Error bars represent the first and last decile.

The first experiment in \Cref{fig:all} (left), there are $M=30$ players, $K=2$ arms, $\mub = (0.8, 0.5)$, $p=0.01$ and $T=10^4$. The optimal assignment is $\Mb^* = (26, 4)$ and note that 
$\nu^* = 0$.
In this example, Cautious Greedy clearly outperforms the other methods as expected when $\nu^*=0$.

The second experiment in \Cref{fig:all} (right) highlights that Cautious Greedy takes a longer time than UCB to assign no player to a suboptimal arm. We have $M=3$ players, $K=2$ arms, $\mub=(0.99, 0.01)$, $p=0.1$ and $T=10^4$. The optimal solution is $\Mb^* = (3, 0)$ so that $\nu^* = 1$. In this example, UCB largely outperforms Cautious Greedy.
In both experiments, ETC incurs a much larger regret, which is consistent with its suboptimal $\Ocal(T^{\frac23})$ regret.

\begin{figure}
  \centering
  \includegraphics[width=.4\textwidth]{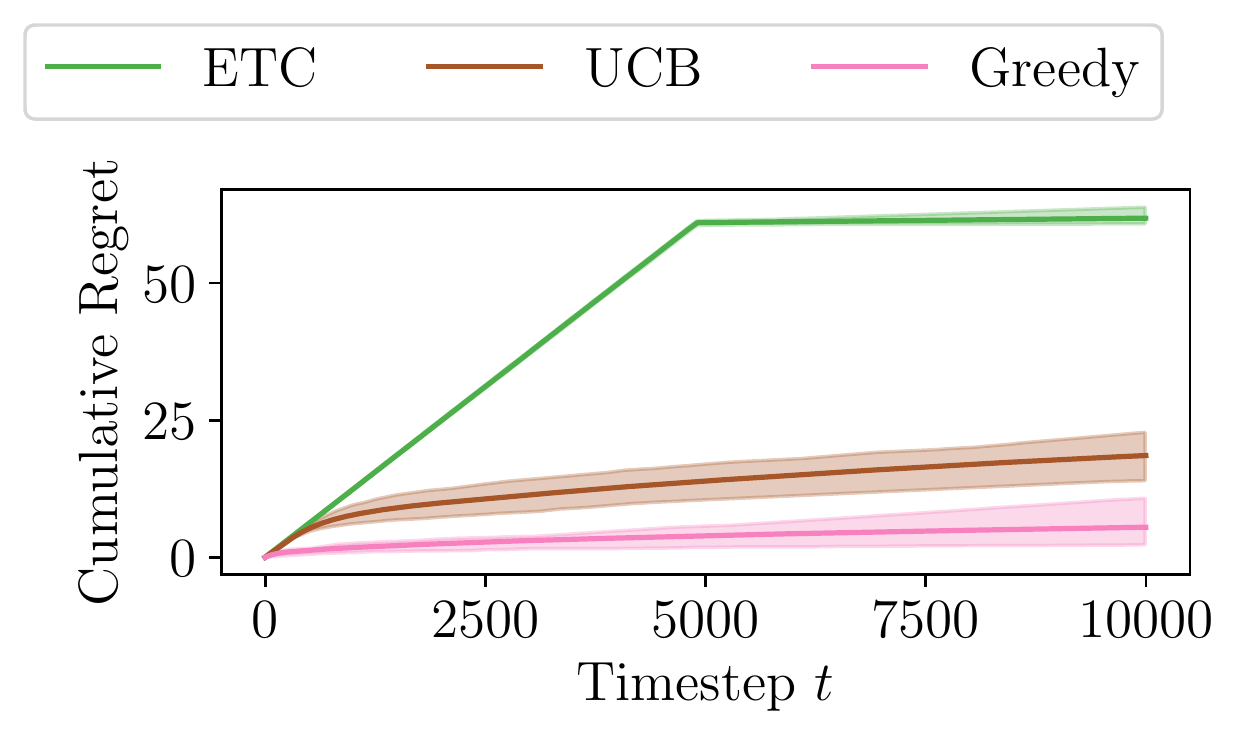}
  \includegraphics[width=.4\textwidth]{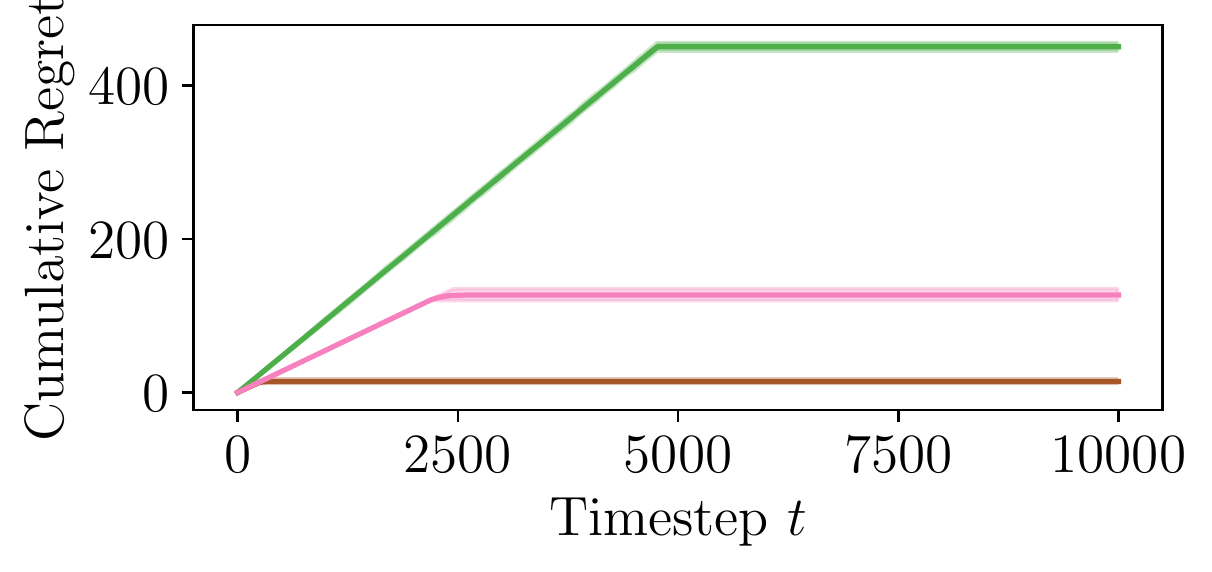}
  \vspace{-1em}
  \caption{\textbf{Benchmark of ETC, UCB and Cautious Cautious Greedy on synthetic data} \textit{(left)} $\nu^* = 0$ \textit{(right)} $\nu^* = 1$.}
  \vspace{-1em}
  \label{fig:all}
\end{figure}

\vspace{-0.2cm}

\section{Conclusion, open problems and future work}
\label{sec:conclusion}
We proposed an asynchronous multiplayer multi-armed bandits algorithm called Cautious Greedy, achieving a regret of order $1/r + \sum_{\nu=1}^{\nu^*} \log(T)/\Delta_{(\nu)}$. In particular, its regret does not scale with $T$ when $\nu^* = 0$. We also prove lower bounds suggesting that the dependency in both $r$ and $\sum_{\nu=1}^{\nu^*} \log(T)/\Delta_{(\nu))}$ is optimal.  A first remark is that the last term cannot be too large, since \Cref{lemma:greedy:necessary condition elimination} shows that $\nu^* = 0$ when the gaps $\Delta_{(j)}$ are small. 

An open question for future work is whether the dependency with respect to other parameters $K, M, p$ can be improved. In particular, the analysis in \cite{degenneAnytimeOptimalAlgorithms2016} shows that in the stochastic setting with full information, Greedy is upper bounded by $\Ocal(\frac{\log(K)}{\Delta})$ when all arms have sub-optimal gap $\Delta$. 
This suggests the dependency w.r.t. $M$ and $K$ of the first term in our bound can be improved.

Our algorithm requires several assumptions to perform properly. Most of them are actually very mild, while others would require an involved analysis to get discarded. 
Without prior knowledge of $T$, a doubling trick~\citep[Theorem 7]{besson2018doubling} can be used when the horizon $T$ is unknown. 
The activation probabilities of players might be heterogeneous in practice. However, the optimization algorithm of \citet{dakdoukMassiveMultiplayerMultiarmed2022} is only optimal in the homogeneous case. No efficient maximization scheme of the problem in \Cref{eq:maxproblem} in the heterogeneous case is currently known. If however we were given access to an oracle maximizing this problem, we believe that our algorithms and their bounds can be adapted. The analysis in the heterogeneous case would not need any change for centralized UCB. Concerning Cautious Greedy, adapting its analysis is more difficult and requires cautious work. 
Also, if $p$ was unknown beforehand, it can easily be estimated on the fly. However, additional errors would come from this estimation and should be handled carefully.

Another significant direction is to go beyond the centralized setting. Being able to handle the decentralized setting where agents are no longer allowed to communicate without cost remains a great challenge and the original motivation of asynchronous multiplayer bandits. A first possibility is to track the number of communications, similarly to \cite{dakdoukMassiveMultiplayerMultiarmed2022}. Although Cautious Greedy's number of communication steps is linear in $T$, a simple low-communication extension would proceed in epochs of doubling size, where communication and updates occur at the end of each epoch. This would make the number of communication steps sub-logarithmic in $T$, with the same regret bound (up to a $2$ factor). 
A second possibility is to directly use collisions to communicate as done for example in~\citep{bistritz2018distributed,boursierSICMMABSynchronisationInvolves2019}. These works however only tackle the synchronized case. 
In our case, communicating through collisions remains possible but the length of communication phases would be significantly increased. In the collision sensing setting, if players $i$ and $j$ need to propagate a bit through collision, they roughly need $\frac{\log(T)}{p^2}$ time-steps to send a single bit with high probability. Whether there exist quicker communication schemes (e.g. using random phase length) for the asynchronous case is an open problem.

\section{Acknowledgments}

Vianney Perchet acknowledges support from the French National Research Agency (ANR) under grant number (ANR-19-CE23-0026 as well as the support grant, as well as from the grant “Investissements d’Avenir” (LabEx Ecodec/ANR-11-LABX-0047).
This work was completed while E. Boursier was a member of TML Lab, EPFL, Lausanne, Switzerland.
\newpage
\bibliography{local-bib}
\bibliographystyle{icml2022}

\newpage
\appendix

\section{Analysis of Cautious Greedy}
\label{app:proof}

\subsection{A useful upper bound}
At many places we will have to bound quantity of the form $\langle \mub - \mub', g(\Mb) - g(\Mb') \rangle$ where $\mub, \mub' \in [0, 1]^K$ and $\Mb, \Mb' \in \mathcal{M}$.
We have
\begin{align*}
  \langle \mub - \mub', g(\Mb) - g(\Mb') \rangle &\leq \langle |\mub - \mub'|, |g(\Mb) - g(\Mb')| \rangle \\
  &\leq \langle |\mub - \mub'|, g(\Mb) + g(\Mb') \rangle \\
  &\leq \sum_{k=1}^K (g(M_k) + g(M_k')) \\
  &\leq \sum_{k=1}^K (M_k + M_k')p \\
  &\leq 2 Mp
\end{align*}

so that we have
\begin{equation}
  \label{eq:rbound}
\langle \mub - \mub', g(\Mb) - g(\Mb') \rangle \leq 2 Mp
\end{equation}

Note that since $M_k \leq \frac1{-\log(1-p)} \leq \frac1{p}$, we have $Mp \leq K$.

\subsection{A precise description of the Round Robin procedure}
\label{sec:RR}
 Rotating $\Ucal$ in a round-robin fashion over $\Ycal \supset \Ucal$ means that $\Ucal$ undergoes one iteration of the Round Robin (RR) procedure.
 See $\Ycal$ as $(y_1, \dots, y_{|\Ycal|})$, $\Ucal$ as $(u_1, \dots, u_s)$. At each iteration, an element from $\Ycal \setminus \Ucal$ is added to $\Ucal$ and an element of $\Ucal$ is dropped in such a way that after $|\Ycal|$ iterations, all elements of $\Ucal$ have been added and dropped from $\Ucal$ exactly once.
 
A possible implementation of the RR procedure is the following. Initialize $\Ucal = (y_1, \dots, y_s)$ and $t=s+1$. Then, performing one iteration of the RR procedure means following \Cref{algo:rr}.

\begin{algorithm}
	\caption{Rotate $\Ucal$ in a round robin fashion over $\Ycal$ (one iteration)}
  \label{algo:rr}
	\begin{algorithmic}[1]
		\STATE {\bfseries Input :} $t$ (iteration number), $\Ucal = (u_1, \dots, u_{|\Ucal|})$, $\Ycal = (y_1, \dots, y_{|\Ycal|})$
        \STATE Remove $u_1$ from $\Ucal$
        \STATE $\forall i \in [|\Ucal| - 1]$, set $u_{i} \leftarrow u_{i + 1}$
        \STATE Set $u_{|\Ucal|} = y_{t \mod |\Ycal|}$
	\end{algorithmic}
\end{algorithm}

\subsection{Proof of \Cref{prop:greedy:up} and \Cref{lemma:third term}}
\subsubsection{Proof of \Cref{lemma:third term}}
\begin{proof}
Assume $\nu^* \geq 1$.
$\Delta^{(\nu^*)}$ is defined as $\Delta^{(\nu^*)} = \langle \mub, g(\Mb^*) - g(\Mb^*_{\nu^*-1}) \rangle$ and $\Delta_{(\nu^*)} = \mu_{(\nu^* + 1)} - \mu_{(\nu^*)}$.

Call $(i)$ the index of the $i$-th worst arm. 
$\Mb^*_{\nu^*-1}$ can be constructed from $\Mb^*_{\nu^*}$.
To do so, remove a player from the arm $j$ such that 
\[
j = \argmin_{i \in \supp(\Mb^*_{\nu^*}), M^*_i \geq 2} \mu_i (g(M^*_i) - g(M^*_{i-1}))
\]
where $M^*_i$ denotes the $i$-th coordinate of $\Mb^*_{\nu^*}$ and place it on arm $(\nu^*)$.

We then have
\[
\Delta^{(\nu^*)} = \mu_j (g(M_j) - g(M_j - 1)) - \mu_{(\nu^*)}p
\]
If $j\neq (\nu^* + 1)$, taking a player from arm $j$ in $\Mb^*_{\nu^*}$ to put it on arm $\nu^* + 1$ would yield to a worse assignment, we have that $\mu_j (g(M_j) - g(M_j - 1)) \geq \mu_{\nu^*+1}(g(M_{(\nu^*+1)}^* + 1)  - g(M_{(\nu^*+1)}^*))$. This inequality is also true if $j = (\nu^*+1)$.

This implies that 
\begin{align*}
	\Delta^{(\nu^*)} &\geq \mu_{\nu^* + 1}(g(M_{(\nu^*+1)}^* + 1)  - g(M_{(\nu^*+1)}^*)) - \mu_{(\nu^*)}p \\
			 &\geq \Delta_{(\nu^*)} (g(M_{(\nu^*+1)}^* + 1)  - g(M_{(\nu^*+1)}^*)) \\
    &\geq \Delta_{(\nu^*)} (g(M)  - g(M-1)) \\
    &= \Delta_{(\nu^*)} p(1 - p)^{M-2}(1 - Mp)
\end{align*}
\end{proof}

\subsubsection{Proof of \Cref{prop:greedy:up}}

The analysis heavily builds upon confidence bounds. We first establish a concentration lemma on the mean reward of each arm.
\begin{lemma}[Concentration of mean rewards]
  \label{lemma:concentration:mu}
  Let $\GOOD$ be the event
  \[
    \forall k \in [K], \forall t \in [T], \enspace |\hat{\mu}_k(t) - \mu_k| \leq \zeta_{kt}
  \]
  Then, $P(\overline{\GOOD}) \leq \frac1{TK}$
\end{lemma}

\begin{proof}[Proof of \Cref{lemma:concentration:mu}]
\label{app:concentration:mu}
Fix $k \in [K]$, by Hoeffding, we have
\[
  P(|\hat{\mu}_k(t) - \mu_k| \geq \epsilon | T_k(t) = \tau) \leq 2 \exp( -2 \tau \epsilon^2)
\]
so that we have
\[
  P(|\hat{\mu}_k(t) - \mu_k| \geq \sqrt{\frac{\log(2T^2K^2)}{2 T_k(t)}} | T_k(t) = \tau) \leq \frac1{K^2T^2}
\]
and with a union bound on $\tau \in [T]$ and a second on $k \in [K]$, we obtain:
\[
  P( \exists k \in [K], |\hat{\mu}_k(t) - \mu_k| \geq \sqrt{\frac{\log(2T^2K^2)}{2 T_k(t)}}) \leq \frac1{KT}
\]

Rearranging, we get with probability $1 - \frac1{TK}$,
\begin{align}
  \forall k \in [K], |\hat{\mu}_k(t) - \mu_k| \leq \sqrt{\frac{\log(2 T^2 K^2)}{2 T_k(t)}}
\end{align}
which is the desired result.
\end{proof}

A consequence of \Cref{lemma:concentration:mu} is that up to a small additive constant in the regret, we can assume that the $\GOOD$ event holds.
\begin{lemma}[Confidence bounds]
\label{lemma:cb:good}
  Define $R_G = R \ind{\GOOD}$, then,
  \begin{equation}
\EE[R] \leq \EE[R_G] + 2
  \end{equation}
\end{lemma}
\begin{proof}[Proof of \Cref{lemma:cb:good}]
    $\EE[R] = \EE[R \mathds{1}_{\mathrm{\GOOD}}] + \EE[R \mathds{1}_{\mathrm{\overline{\GOOD}}}]$ and $R \mathds{1}_{\mathrm{\overline{\GOOD}}} \leq 2KT \mathds{1}_{\mathrm{\overline{\GOOD}}}$, we then conclude from \Cref{lemma:concentration:mu}.
\end{proof}

Working under the GOOD event makes the analysis much easier.
We begin by showing that $\nu$ is a lower bound on the optimal number of arms to eliminate:
\begin{lemma} Under the $\GOOD$ event, $\nu \leq \nu^*$ at any time~$t$.\label{lemma:mu:ub}
\end{lemma}

\begin{proof}[Proof of \Cref{lemma:mu:ub}]
\label{app:lemma:mu:ub}
$\nu$ is only increased in the while loop. We want to show that if $\nu = \nu^*$, then the condition in the while loop cannot be met.
  Assume by contradiction that $\nu = \nu^*$ and  $\max_{ \Mb \in \mathcal{M}} \langle \hat{\mub}^L, g(\Mb) \rangle > \max_{ \Mb \in \mathcal{M}_{\nu}} \langle \hat{\mub}^H, g(\Mb) \rangle$.
  By the good event, we have
  \begin{align*}
\max_{ \Mb \in \mathcal{M}} \langle \hat{\mub}^L, g(\Mb) \rangle < \max_{ \Mb \in \mathcal{M}} \langle \mub, g(\Mb) \rangle
  \end{align*}
  and
  \begin{align*}
\max_{ \Mb \in \mathcal{M}_\nu} \langle \hat{\mub}^H, g(\Mb) \rangle > \max_{ \Mb \in \mathcal{M}_\nu} \langle \mub, g(\Mb) \rangle
  \end{align*}
  Therefore

  \begin{align*}
    &\max_{ \Mb \in \mathcal{M}} \langle \hat{\mub}^L, g(\Mb) \rangle > \max_{ \Mb \in \mathcal{M}_{\nu}} \langle \hat{\mub}^H, g(\Mb) \rangle \\
    &\implies \max_{ \Mb \in \mathcal{M}} \langle \mub, g(\Mb) \rangle > \max_{ \Mb \in \mathcal{M}_{\nu}} \langle \mub, g(\Mb) \rangle \\
  \end{align*}
  and since $\nu = \nu^*$,
  \[ \max_{ \Mb \in \mathcal{M}_{\nu}} \langle \mub, g(\Mb) \rangle = \max_{ \Mb \in \mathcal{M}} \langle \mub, g(\Mb) \rangle. \]
  This yields the following contradiction:
  \begin{align*}
    \max_{ \Mb \in \mathcal{M}} \langle \mub, g(\Mb) \rangle > \max_{ \Mb \in \mathcal{M}} \langle \mub, g(\Mb) \rangle.
  \end{align*}
\end{proof}


More generally, under the $\GOOD$ event, Cautious Greedy never eliminates an optimal arm:
\begin{lemma}[Optimal arms are never eliminated]
\label{lemma:no good arms eliminated}
Under the $\GOOD$ event, the set of optimal arms is always included in the set of active arms:  $ \support(\Mb^*) \subseteq \mathcal{K}(t)$
\end{lemma}

\begin{proof}[Proof of \Cref{lemma:no good arms eliminated}]
\label{app:lemma:no good arms eliminated}
  Elimination may happen when the set of active arms is updated. An arm $k$ is eliminated at this stage if
  $\hat{\mu}_k^H < \mu_{(\nu + 1)}^L$.
  But since $\nu \leq \nu^*$, this implies $\hat{\mu}_k^H < \mu_{(\nu^* + 1)}^L$.
  Under the good event  $\hat{\mu}_k^H > \mu_k$ and $\mu_{(\nu^* + 1)}^L < \mu_{(\nu^* + 1)}$ so that
$\hat{\mu}_k^H < \mu_{(\nu^* + 1)}^L$ implies $\mu_k < \mu_{(\nu^* + 1)}$ and therefore $k \notin  \mathcal{E}_{\nu^*}^*$.
\end{proof}


Since Cautious Greedy never eliminates any optimal arm and since $\nu$ increases,  $\nu$ will eventually reach $\nu^*$ and bad arms will no longer remain. But as long as $\nu < \nu^*$, Cautious Greedy will pay a non-zero cost. This source of error as well as others strongly depends on the number of times arms are pulled without collisions. Indeed, as the number of pulls without collision increases, the reward estimates $\hat{\mub}$ become more accurate, making Cautious Greedy's decisions better. Therefore, we introduce $q(t) = \min_{k \in \mathcal{K}(t)} T_k(t)$, the number of times each active arm has been played without collision.

To be able to understand how $q(t)$ scales with $t$, a pre-requisite is to count the number of times that arms are assigned at least one player. Denote $\tau_k(t)$ the number of times arm $k$ has been assigned at least one player at time $t$ and $\tau(t) = \min_{k \in \mathcal{K}(t)} \tau_k(t)$. The next Lemma exhibits a lower bound on $\tau(t)$:
\begin{lemma}[ Scaling of $\tau$ with $t$ ]
  \label{lemma:tau}
  We have $\tau(t) \geq \max(\frac{t}{\nu^* + 1} - \nu^*, 0)$. Furthermore, if the condition Line 8 in \Cref{algo:greedy} is satisfied, we have
  $\tau(t-1) \geq \frac{t-1}{\nu^* + 1}$.
\end{lemma}
\begin{proof}[Proof of \Cref{lemma:tau}]
\label{app:lemma:tau}
  Call $t_n$ the value of $t$ the $n$-th time where $t = 0 \mod |\mathcal{U}|$. Between $t_{n}$ and $t_{n+1} - 1$ (included) all arms have been played $|\mathcal{U}_n| - u_n$ times where $\mathcal{U}_n$ and $u_n$ are the set of active but not yet accepted arms $U$ and the number of arms under pressure $u$ after the updates at time $t=t_{n}$.
 $\tau$ increases linearly between time $t_n$ and $t_{n+1}$ except for $u_n$ time steps where $u_n = \nu_n - |[K] \setminus \mathcal{K}|$ is the number of arms that need to be put under pressure during phase $n$ but that are not yet eliminated and $\nu_n$ is the value of $\nu$ during phase $n$.

  We have that for $t_n \leq t < t_{n+1}$ :
  \begin{align*}
    \tau(t) &\geq \tau(t_n -1) + \max(t - (t_n - 1) - u_n, 0) \\
    &= \tau(t_n - 1) + \max((t - (t_n - 1))\frac{t_{n+1} - (t_n - 1) - u_n}{t_{n +1} - (t_n - 1)} - (t_{n+1} - t) \frac{u_n}{t_{n +1} - (t_n - 1)}, 0) \\
  \end{align*}

  and
  \begin{align*}
    \tau(t_{n+1} - 1) - \tau(t_{n}-1) &= t_{n+1} - (t_{n} - 1) - u_{n} \\
&= (t_{n + 1} - (t_{n} - 1))\frac{t_{n+1} - (t_{n} - 1) - u_{n}}{t_{n + 1} - (t_{n} - 1)}
  \end{align*}

  Since $t_{n+1} - (t_n - 1) = |\mathcal{K}_n \setminus \mathcal{A}_n|$ we have:
  \[
\frac{t_{n+1} - (t_n - 1) - u_n}{t_{n+1} - (t_n - 1)}  = \frac{|\mathcal{K}_n| - |\mathcal{A}_n| - u_n }{|\mathcal{K}_n| - |\mathcal{A}_n|} \geq \frac1{u_n + 1} \geq  \frac1{\nu^* + 1}
  \]

  It follows that for all $n \geq 1$,
  \[
\tau(t_{n} - 1) \geq \frac{t_{n} - 1}{\nu^* + 1}
  \]

  Therefore, we obtain $t_n \leq t \leq t_{n+1}$
  \begin{align*}
    \tau(t) &\geq \frac{t_n - 1}{\nu^* + 1} +  \max((t - (t_n - 1))\frac{1}{\nu^* + 1} - (t_{n+1} - t) \frac{u_n}{t_{n +1} - (t_n - 1)}, 0) \\
&\geq \frac{t_n - 1}{\nu^* + 1} +  \max((t - (t_n - 1))\frac{1}{\nu^* + 1} - \nu^*, 0) \\
&\geq  \max(\frac{t}{\nu^* + 1} - \nu^*, \frac{t_n - 1}{\nu^* + 1}) \\
&\geq  \max(\frac{t}{\nu^* + 1} - \nu^*, 0) \\
  \end{align*}
  Since this last line holds for all $n$, we have for any $t$ that $\tau(t) \geq \max(\frac{t}{\nu^* + 1} - \nu^*, 0)$.
\end{proof}


The next step is to link $\tau(t)$ to $q(t)$ by taking collisions into account. By noting that $g(M_k) \geq p$, a first easy observation is that $\EE[T_k \mid \tau_k] \geq \tau_k p$. We, therefore, expect $q$ to scale approximately with $p \tau$.
The next Lemma shows a more precise statement:
\begin{lemma}
  \label{lemma:q}
  Define $R_{q} = \sum R_{q,t}$ where $R_{q, t} = R_t \ind{\GOOD} \ind{q(t) \geq \frac13 p \tau(t)}$,
  \begin{align*}
    \EE[R_G] =  \EE[R_q] + 11 (\nu^* + 1)  KM
  \end{align*}
\end{lemma}

\begin{proof}[Proof of \Cref{lemma:q}]
\label{app:lemma:q}
  Rename $T_k(t) = T_k(\tau(t))$, to make the dependence on $\tau$ obvious. We have the lower bound $\EE[T_k(\tau(t))] \geq p \tau(t)$.

We can write
\begin{align*}
  \EE[R_G] &= \EE[\sum_{t=1}^{(\nu^* + 1)^2} R_{G, t} + \sum_{t=(\nu^* + 1)^2}^{T} R_{G, t}] \\
&\leq (\nu^* + 1)^2 2Mp + \EE[\sum_{t=(\nu^* + 1)^2}^{T} R_{G, t}] \tag{By \Cref{eq:rbound}} \\
\end{align*}
Then for $R_t\defeq\sum_{k=1}^K  \EE[\eta_k(\Mb^*) X_k] - \EE[\eta_k^t(\Mb(t)) X_k^t]$,
\begin{align*}
&\EE[\sum_{t=(\nu^* + 1)^2}^{T} R_{G, t}] \\
         &= \EE[\sum_{t=(\nu^* + 1)^2}^T R_t (\ind{\exists k \in \mathcal{K}(t), T_k(\tau(t)) < (1 - \rho)\EE[T_k(\tau(t)) | \Mb_{[1:t]}]} \\ 
         &+ \ind{\forall k \in \mathcal{K}(t), T_k(\tau(t)) \geq (1 - \rho)\EE[T_k(\tau(t)) | \Mb_{[1:t]}]})] \\
         &\leq \underbrace{\EE[\sum_{t=(\nu^* + 1)^2}^T 2 Mp \EE[\ind{\exists k \in \mathcal{K}(t), T_k(\tau(t)) < (1 - \rho)\EE[T_k(\tau(t)) | \Mb_{[1:t]}]} | \Mb_{[1:t]}]]}_{(i)} \tag{By \Cref{eq:rbound}} \\ & \enspace \enspace \enspace+ \underbrace{\sum_{t=(\nu^* + 1)^2}^T \EE[R_t\ind{\forall k \in \mathcal{K}(t), T_k(\tau(t)) \geq (1 - \rho)\EE[T_k(\tau(t)) | \Mb_{[1:t]}]}]}_{(ii)} \\
  \end{align*}

  Bounding (i):
  \begin{align*}
    (i) &\leq \EE[\sum_{t= (\nu^* + 1)^2}^T 2 Mp \sum_{k=1}^{\mathcal{K}(t)}  P(T_k(\tau(t)) < (1 - \rho)\EE[T_k(\tau(t)) |  \Mb_{[1:t]}] | \Mb_{[1:t]})] \\
         &\leq \EE[\sum_{t=(\nu^* + 1)^2}^T 2 Mp \sum_{k=1}^{\mathcal{K}(t)}  \EE[\exp(-\frac{\rho^2}{2}\EE[T_k(\tau(t)) |  \Mb_{[1:t]}]) | \Mb_{[1:t]}]] \\
         &\leq \EE[\sum_{t=(\nu^* + 1)^2}^T 2KMp  \exp(-\frac{\rho^2}{2}p \tau(t)) ] \\
         &\leq \sum_{t=1}^{+\infty} 2KMp  \exp(-\frac{\rho^2}{2}p\frac{t}{\nu^* + 1}) \tag{By  \Cref{lemma:tau}}\\
         &\leq 2KMp \frac{2(\nu^* + 1)}{\rho^2 p} \\
         &= 2KM \frac{2(\nu^* + 1)}{\rho^2}
\end{align*}

Bounding (ii)
\begin{align*}
  (ii) &= \sum_{t=1}^T \EE[R_t\ind{\forall k \in \mathcal{K}(t), T_k(\tau(t)) \geq (1 - \rho)\EE[T_k(\tau(t)) | \Mb_{[1:t]}]}] \\
  &\leq \sum_{t=1}^T \EE[R_t\ind{\forall k \in \mathcal{K}(t), T_k(\tau(t)) \geq (1 - \rho)p \tau(t)}] \\
\end{align*}
Setting $\rho= \frac23$ gives

\[
\EE[R_G] \leq \EE[R_q] + 2(\nu^* + 1)^2 Mp + 9KM (\nu^* + 1) \leq \EE[R_q] + 11 KM (\nu^* + 1)
\]

\end{proof}


We can now focus on upper-bounding the different sources of errors. First, $\EE[R_q]$ can trivially be written as:
\[
\EE[R_q] = \EE[R_\nu] + \EE[R_{\mathcal{E}}] + \EE[R_\Mb]
\]
\begin{align*}
\text{where } &R_\nu = \sum_{t=1}^T \langle \mub,  g(\Mb^*) - g(\Mb^*_{\nu}) \rangle,\\
&R_{\mathcal{E}} = \sum_{t=1}^T \langle \mu,  g(\Mb^*_{\nu}) - g(\Mb^*_{\mathcal{E}(t)}) \rangle\\
\text{and }&R_{\Mb} = \sum_{t=1}^T \langle \mu, g(\Mb^*_{\mathcal{E}(t)}) - g(\Mb(t)) \rangle.
\end{align*}
$\Mb^*_\nu = \Mb^{\mub}_{\Mcal_{\nu}}$ is the optimal assignment of players when at most $\nu$ arms can be assigned zero players and $\Mb^*_{\mathcal{E}(t)}= \Mb^{\mub}_{\Mcal_{\mathcal{E}(t)}}$ is the optimal assignment of players when only arms not in $\mathcal{E}(t)$ can be assigned zero players.
Here and in the rest of the analysis, we have dropped the factors $\ind{\GOOD} \ind{q(t) \geq \frac13 p \tau(t)}$ to simplify the notations.

These three terms measure a different aspect of the regret:  $R_\nu$ measures the error due to $\nu$ the number of arms under pressure being different from $\nu^*$ the optimal number of players to eliminate,  $R_{\mathcal{E}}$ measures the error due to $\mathcal{E}(t)$ being different from $\support(\Mb^*_\nu)$ the optimal set of arms that must be assigned at least one player when up to $\nu$ players can be assigned zero players and $R_{\Mb}$ measures the error due to $\Mb(t)$ being different from $\Mb^*_{\mathcal{E}(t)}$ the optimal assignment of players  among possible assignments in $\mathcal{M}_{\mathcal{E}(t)}$.

Let us start with the first term $R_\nu$. As the number of samples seen increases, $\nu$ increases to get closer to $\nu^*$. The following Lemma provides a maximum on the number of samples seen before the algorithm detects that $\nu$ should increase.

\begin{lemma}[Number of iterations before $\nu$ increases]
  \label{lemma:greedy:nu}
  If each active arm has been played without collision at least $q$ times with $q \geq q_{k} = \frac{8 M^2p^2 \log(2K^2 T^2)}{(\Delta^{(k)})^2}$, then $\nu \geq k$.
\end{lemma}

\begin{proof}[Proof of \Cref{lemma:greedy:nu}]
\label{app:greedy:nu}
Call $\Mb^{*, H}_{\nu} = \argmax_{ \Mb \in \mathcal{M}_{\nu}} \langle \mub^H, g(\Mb) \rangle$.
 \begin{align*}
   &\max_{ \Mb \in \mathcal{M}} \langle \mub^L, g(\Mb) \rangle >  \langle \mub^H, g(\Mb^{*, H}_{\nu}) \rangle \\
   & \impliedby \langle \mub^L, g(\Mb^*) \rangle  > \langle \mub + 2\zetab, g(\Mb^{*, H}_{\nu}) \rangle  \tag{By the GOOD event}\\
&\impliedby \langle \mub - 2\zetab, g(\Mb^*) \rangle  > \langle \mub, g(\Mb^*_{\nu}) \rangle + 2\langle \zetab, g(\Mb^{*, H}_{\nu}) \rangle \tag{By the GOOD event and optimality of $\Mb^*_{\nu}$} \\
  &\impliedby \langle \mub , g(\Mb^*) - g(\Mb_{\nu}^*) \rangle  > 2\langle \zetab, g(\Mb^*) + g(\Mb^{*, H}_{\nu}) \rangle \\
  & \impliedby  \langle \mub, g(\Mb^*) - g(\Mb_{\nu}^*) \rangle > 4 Mp \max_{k \in \mathcal{K}} \zeta_k  \tag{By \Cref{eq:rbound}}\\
  & \impliedby  \langle \mub, g(\Mb^*) - g(\Mb_{\nu}^*) \rangle > 4 Mp \max_{k \in \mathcal{K}} \sqrt{\frac{\log(2 T^2 K^2)}{2 T_k(t)}} \\
  & \impliedby  \frac{8 M^2 p^2 \log(2K^2 T^2)}{(\langle \mub, g(\Mb^*) - g(\Mb_{\nu}^*) \rangle)^2} <  \min_{k \in \mathcal{K}} T_k \\
  & \impliedby  \frac{8 M^2 p^2 \log(2K^2 T^2)}{(\langle \mub, g(\Mb^*) - g(\Mb_{\nu}^*) \rangle)^2} <  q
\end{align*}
\end{proof}


Note that as long as $\nu \leq \nu^*$, $R_\nu$ increases by $\langle \mub,  g(\Mb^*) - g(\Mb^*_{\nu}) \rangle$. \Cref{lemma:greedy:nu} then allows to bound $R_\nu$: \begin{lemma}[Bound on $R_\nu$]
  \label{lemma:greedy:bound:nu}
  \begin{align*}
    &\EE[R_\nu] \leq 72 M\min(Mp, K) (\nu^* + 1)\frac{\log(2K^2 T^2)}{\Delta^{(\nu^*)}}
  \end{align*}
\end{lemma}

\begin{proof}[Proof of \Cref{lemma:greedy:bound:nu}]
\label{app:greedy:bound:nu}
  Call $t_v$ the last time that $\nu(t) = \nu$ and set $t_{\nu^*} = T + 1$ and $t_{-1}$ = 0. Note that $t_\nu + 1$ necessarily verifies $t = 0 \mod |\mathcal{U}|$ so that $\tau(t_\nu) \geq \frac{t_\nu}{\nu^* + 1}$ according to \Cref{lemma:tau}.
  \begin{align*}
    \EE[R_\nu] &= \EE[\sum_{t=1}^T \langle \mub,  g(\Mb^*) - g(\Mb^*_{\nu(t)}) \rangleƒ] \\
          &= \EE[\sum_{\nu=0}^{\nu^*} \sum_{t = t_{\nu - 1} + 1}^{t_{\nu}} \underbrace{\langle \mub,  g(\Mb^*) - g(\Mb^*_{\nu}) \rangle}_{A_\nu}] \\
          &=  \EE[\sum_{\nu=0}^{\nu^*} (t_{\nu} - t_{\nu - 1}) A_\nu] \\
          &= \EE[\sum_{\nu=0}^{\nu^*} t_{\nu} A_\nu - \sum_{\nu=0}^{\nu^*} t_{\nu - 1} A _\nu] \\
          &= \EE[\sum_{\nu=1}^{\nu^*} t_{\nu - 1} (A_{\nu - 1} - A _\nu) - t_{-1} A _0 + t_{\nu^*} A _{\nu^*}] \\
          &= \EE[\sum_{\nu=1}^{\nu^*} t_{\nu - 1} (A_{\nu - 1} - A _\nu) ]\\
          &\leq \EE[\sum_{\nu=1}^{\nu^*} (\nu^* + 1) \tau(t_{\nu-1}) (A_{\nu - 1} - A _\nu)] \tag{By \Cref{lemma:tau}} \\
          &\leq \EE[\sum_{\nu=1}^{\nu^*} \frac{3(\nu^* + 1)}{p} q_{\nu - 1} (A_{\nu - 1} - A _\nu)] \tag{By \Cref{lemma:q}} \\
          &=  \frac{3(\nu^* + 1)}{p} \sum_{\nu=1}^{\nu^*} \frac{8(Mp)^2 \log(2K^2 T^2) \langle \mub, g(\Mb_{\nu}^*) - g(\Mb_{\nu - 1}^*) \rangle}{ \big(\langle \mub, g(\Mb^*) - g(\Mb_{\nu - 1}^*) \rangle \big)^2} \tag{By \Cref{lemma:greedy:nu}}\\
          &=  24 M^2 p \log(2K^2 T^2)(\nu^* + 1) \sum_{\nu=1}^{\nu^*}  \underbrace{\bigg( \frac1{\langle \mub, g(\Mb^*) - g(\Mb_{\nu - 1}^*) \rangle} - \frac{\langle \mub, g(\Mb^*) - g(\Mb_{\nu}^*) \rangle}{ \big(\langle \mub, g(\Mb^*) - g(\Mb_{\nu - 1}^*) \rangle \big)^2} \bigg)}_{\defeq l_\nu} \\
          &=  24 M^2 p \log(2K^2 T^2)(\nu^* + 1) \bigg[ \frac1{\langle \mub, g(\Mb^*) - g(\Mb_{\nu - 1}^*) \rangle} \\ 
          &+ \sum_{\nu=1}^{\nu^*-1}  \underbrace{\bigg( \frac1{\langle \mub, g(\Mb^*) - g(\Mb_{\nu - 1}^*) \rangle} - \frac{\langle \mub, g(\Mb^*) - g(\Mb_{\nu}^*) \rangle}{ \big(\langle \mub, g(\Mb^*) - g(\Mb_{\nu - 1}^*) \rangle \big)^2} \bigg)}_{\defeq l_\nu} \bigg]
  \end{align*}
  where $q_\nu$ is given in \ref{lemma:greedy:nu} and we used $A _{\nu^*} = 0$. From there, we have the following inequalities
  \begin{align*}
     \sum_{\nu=1}^{\nu^*-1} l_\nu & \leq \sum_{\nu=1}^{\nu^*-1}  \bigg( \frac1{\langle \mub, g(\Mb^*) - g(\Mb_{\nu}^*) \rangle} - \frac{\langle \mub, g(\Mb^*) - g(\Mb_{\nu}^*) \rangle}{ \big(\langle \mub, g(\Mb^*) - g(\Mb_{\nu - 1}^*) \rangle \big)^2} \bigg) \\
      & = \sum_{\nu=1}^{\nu^*-1}  \bigg( \frac1{\langle \mub, g(\Mb^*) - g(\Mb_{\nu}^*) \rangle} - \frac{1}{\langle \mub, g(\Mb^*) - g(\Mb_{\nu - 1}^*) \rangle} \bigg)\big( 1+\frac{\langle \mub, g(\Mb^*) - g(\Mb_{\nu}^*) \rangle}{\langle \mub, g(\Mb^*) - g(\Mb_{\nu-1}^*) \rangle}  \big)\\
      & \leq 2\sum_{\nu=1}^{\nu^*-1}  \bigg( \frac1{\langle \mub, g(\Mb^*) - g(\Mb_{\nu}^*) \rangle} - \frac{1}{\langle \mub, g(\Mb^*) - g(\Mb_{\nu - 1}^*) \rangle} \bigg) \\
      & \leq \frac2{\langle \mub, g(\Mb^*) - g(\Mb_{\nu^*-1}^*) \rangle} \end{align*}
  
  We conclude the proof by using $Mp \leq \min(Mp, K)$.
\end{proof}


We now focus on the second term $R_{\mathcal{E}}$. For a given $\nu$, two things may prevent a sub-optimal choice of arms $\mathcal{E}$ on which at least one player must be assigned. Either an arm in $\mathcal{E}$ is eliminated or an arm in $[K] \setminus \Ecal$ is accepted. \Cref{lemma:greedy:arm elimination} shows a condition under which a sub-optimal arm~$i$ is eliminated:
\begin{lemma}[Number of samples seen before a sub-optimal arm is eliminated]
  \label{lemma:greedy:arm elimination}
  Fix $\nu$, let $\mathcal{E}^*_\nu =\support(\Mb^*_\nu)$ and let $i \notin \mathcal{E}^*_\nu$ be a sub-optimal arm. When each arm has been played without collision at least $q$ times with
\[
q \geq q_{E, i} = \frac{8\log(2 T^2 K^2)}{(\mu_{(\nu+1)} - \mu_i)^2}
\]
  then arm $i$ has necessarily been eliminated.
\end{lemma}

\begin{proof}[Proof of \Cref{lemma:greedy:arm elimination}]
\label{app:greedy:arm elimination}
  Since $i \notin \mathcal{E}^*_\nu$, $\mu_i < \mu_{(\nu + 1)}$, the algorithm notices that $i$ must be eliminated if
  \begin{align*}
    &\mu_i^H < \mu_{(\nu + 1)}^L \\
    &\impliedby \mu_i + 2 \zeta_i < \mu_{(\nu + 1)} - 2 \zeta_{(\nu+1)} \\
    &\impliedby \zeta_i + \zeta_{(\nu+1)}  < \frac{\mu_{(\nu + 1)} - \mu_i}{2}  \\
    &\impliedby 2\sqrt{\frac{\log(2 T^2 K^2)}{2 q}}  < \frac{\mu_{(\nu + 1)} - \mu_i}{2}  \\
    &\impliedby q  > \frac{8\log(2 T^2 K^2)}{(\mu_{(\nu+1)} - \mu_i)^2}
  \end{align*}
\end{proof}

\Cref{lemma:greedy:arm acceptation} shows a condition under which an optimal arm~$j$ is accepted:
\begin{lemma}[Number of samples seen before an optimal arm is accepted]
  \label{lemma:greedy:arm acceptation}
  Fix $\nu$, let $\mathcal{E}^*_\nu = \support(\Mb^*_E)$ and let $j \in \mathcal{E}^*_\nu$ an optimal arm. When each arm has been played without collision at least $q$ times with
\[
q \geq q_{A, i} = \frac{8\log(2 T^2 K^2)}{(\mu_j - \mu_{(\nu)})^2}
\]
  then arm $j$ has necessarily been accepted.
\end{lemma}

\begin{proof}[Proof of \Cref{lemma:greedy:arm acceptation}]
\label{app:greedy:arm acceptation}
  Since $j \in \mathcal{E}^*_\nu$, $\mu_j > \mu_{(\nu)}$, the algorithm notices that $j$ must be accepted if
  \begin{align}
    &\mu_{(\nu)}^H < \mu_{j}^L \\
    &\impliedby \mu_{(\nu)} + 2 \zeta_{(\nu)} < \mu_j - 2 \zeta_j \\
    &\impliedby \zeta_{(\nu)} + \zeta_j  < \frac{\mu_j - \mu_{(\nu)}}{2}  \\
    &\impliedby 2\sqrt{\frac{\log(2 T^2 K^2)}{2 q}}  < \frac{\mu_j - \mu_{(\nu)}}{2}  \\
    &\impliedby q  > \frac{8\log(2 T^2 K^2)}{(\mu_j - \mu_{(\nu)})^2}
  \end{align}
\end{proof}

The two previous lemmas allow to quantify when arms are accepted or rejected. The next lemma measures the cost of choosing a sub-optimal set of arms on which at least one player must be assigned.
\begin{lemma}[Cost of choosing a sub-optimal $\Ecal$]
  \label{lemma:greedy:cost:e}
  Let $\mathcal{E}$ a set of arms of size $K - \nu$ such that $\mathcal{E} \neq \mathcal{E}^*_\nu = \support(\Mb^*_\nu)$. Then, we have:
  \begin{align*}
    \langle \mub, g(\Mb^*_{\nu}) - &g(\Mb^*_{\mathcal{E}}) \rangle \leq \\& p( \sum_{i \in \mathcal{E} \setminus \mathcal{E}_{\nu}^*} \mu_{(\nu+1)} -  \mu_{i} +   \sum_{j \in \mathcal{E}_{\nu}^* \setminus \mathcal{E}} \mu_j - \mu_{(\nu)})
  \end{align*}
\end{lemma}

\begin{proof}[Proof of \Cref{lemma:greedy:cost:e}]
\label{app:greedy:cost:e}
  Let $\mathcal{E} \neq \mathcal{E}_{\nu}^*$ and define indexes $i_1, \dots, i_n$ by
  \[ \mathcal{E} \setminus \mathcal{E}_{\nu}^* = \{ i_1, \dots, i_n \}\] and indexes $j_1, \dots, j_n$ by \[
\mathcal{E}_\nu^* \setminus \mathcal{E} = \{ j_1, \dots, j_n \}
  \]

  We now construct $\Mb_{\mathcal{E}}$. Arms that are in $\mathcal{E}$ but not in $\mathcal{E}_\nu^*$ are assigned $1$ player the corresponding players are taken from arms in $\mathcal{E}_\nu^*$ but not in $\mathcal{E}$. Formally
  \[ \forall k \in [n], \Mb_{\mathcal{E}}[i_k] = 1 \]
  and
  \[ \forall k \in [n], \Mb_{\mathcal{E}}[j_k] = \Mb^*[j_k] - 1 \]
  and other arms are untouched:
  \[ \forall k \in \mathcal{E}^*_\nu \cap \mathcal{E}, \Mb_{\mathcal{E}}[k] = \Mb^*[k] \]

  The cost is given by:
\begin{align*}
    \langle \mub, g(\Mb^*_{\nu}) - g(M^*_{\mathcal{E}}) \rangle&\leq \langle \mub, g(\Mb^*_{\nu}) -g(\Mb_{\mathcal{E}}) \rangle \\
    &= \sum_{k=1}^n (\mu_{j_k} \sbr{g(\Mb^*[j_k]) - g(\Mb^*[j_k] - 1)}  - \mu_{i_k}p) \\
    &\leq \sum_{k=1}^n (\mu_{j_k}  - \mu_{i_k})p \\
    &\leq \sum_{k=1}^n (\mu_{j_k} - \mu_{(\nu)} + \mu_{(\nu+1)} -  \mu_{i_k})p \\
    &\leq p( \sum_{i \in \mathcal{E} \setminus \mathcal{E}_{\nu}^*} \mu_{(\nu+1)} -  \mu_{i} +   \sum_{j \in \mathcal{E}_{\nu}^* \setminus \mathcal{E}} \mu_j - \mu_{(\nu)})
  \end{align*}
\end{proof}

We can now bound $R_{\mathcal{E}}$:
\begin{lemma}[Bound on $R_{\mathcal{E}}$]
  \label{lemma:greedy:bound:e}
  \begin{equation*}
\EE[R_{\mathcal{E}}] \leq \frac{120 \log(2K^2T^2)}{\Delta^{(\nu^*)}}+ \sum_{i=1}^{\nu^*}  \frac{120 \log(2 T^2 K^2)}{\Delta_{(i)}} 
  \end{equation*}
\end{lemma}

\begin{proof}[Proof of \Cref{lemma:greedy:bound:e}]
\label{app:greedy:bound:e}
  Call $t_v$ the last time that $\nu(t) = \nu$ and set $t_{\nu^*} = T + 1$ and $t_{-1} = 0$.
  We can write
  \begin{align*}
    R_{\mathcal{E}} &= \sum_{t=1}^T \langle \mub,  g(\Mb^*_{\nu(t)}) - g(\Mb^*_{\mathcal{E}(t)}) \rangle
  \end{align*}

  \begin{align*}
          &= \sum_{\nu=0}^{\nu^*} \sum_{t = t_{\nu - 1} + 1}^{t_{\nu}} \langle \mub,  g(\Mb^*_{ \nu}) - g(\Mb^*_{\mathcal{E}(t)}) \rangle \\
    & \leq \sum_{\nu=0}^{\nu^*} \sum_{t = t_{\nu - 1} + 1}^{t_{\nu}} p( \sum_{i \in \mathcal{E}(t) \setminus \mathcal{E}_{\nu}^*} (\mu_{(\nu+1)} -  \mu_{i}) +   \sum_{j \in \mathcal{E}_{\nu}^* \setminus \mathcal{E}(t)} (\mu_j - \mu_{(\nu)})) \tag{Using \Cref{lemma:greedy:cost:e}}\\
                    &= \underbrace{\sum_{\nu=0}^{\nu^*} \sum_{t = t_{\nu - 1} + 1}^{t_{\nu}} p \sum_{i \in \mathcal{E}(t) \setminus \mathcal{E}_{\nu}^*} (\mu_{(\nu+1)} -  \mu_{i})}_{(i)} +  \underbrace{\sum_{\nu=0}^{\nu^*} \sum_{t = t_{\nu - 1} + 1}^{t_{\nu}} p \sum_{j \in \mathcal{E}_{\nu}^* \setminus \mathcal{E}(t)} (\mu_j - \mu_{(\nu)})}_{(ii)}
  \end{align*}

  Let us cut the execution of the algorithms in phases where phase $n$ starts when it is the $n$-th time that the condition Line~\ref{phase:condition} in \Cref{algo:greedy} is satisfied. Note again that updates of $\mathcal{A}$, $\mathcal{K}$, and $\nu$ occur at the beginning of each phase.
  Denote $\mathcal{N}_{\nu}$ the phases between $t_{\nu-1} + 1$ and $t_{\nu}$.

  \paragraph{Bounding (i)} Denote $\tau_n$ the number of pulls of active arms at the end of phase $n$.
  \begin{align*}
    (i) &\leq \sum_{\nu=0}^{\nu^*} \sum_{t = t_{\nu - 1} + 1}^{t_{\nu}} p \sum_{i \notin \mathcal{E}_{\nu}^*} (\mu_{(\nu+1)} -  \mu_{i})\ind{i \in \mathcal{E}(t)} \\
    &= \sum_{\nu=1}^{\nu^*} \sum_{t = t_{\nu - 1} + 1}^{t_{\nu}} p \sum_{i \notin \mathcal{E}_{\nu}^*} (\mu_{(\nu+1)} -  \mu_{i})\ind{i \in \mathcal{E}(t)} \\
    &= \sum_{\nu=1}^{\nu^*} \sum_{n \in N_\nu} \sum_{t = t_{\nu - 1} + 1}^{t_{\nu}} p \sum_{i \notin \mathcal{E}_{\nu}^*} (\mu_{(\nu+1)} -  \mu_{i})\ind{i \in \mathcal{E}(t)} \ind{ \text{ $t$ belong to phase $n$ }} \\
    &\leq \sum_{\nu=1}^{\nu^*} \sum_{n \in N_\nu} p \sum_{i \notin \mathcal{E}_{\nu}^*} (\mu_{(\nu+1)} -  \mu_{i}) (\text{ Number of times arm $i$ is pulled during phase $n$ }) \\
    &= \sum_{\nu=1}^{\nu^*} \sum_{n \in N_\nu} p \sum_{i=1}^\nu (\mu_{(\nu+1)} -  \mu_{(i)}) (\text{ Number of times arm $(i)$ is pulled during phase $n$ }) \\
    &= p \sum_{i=1}^{\nu^*} \underbrace{\sum_{\nu=i}^{\nu^*} \sum_{n \in N_\nu} (\mu_{(\nu+1)} -  \mu_{(i)}) (\text{ Number of times arm $(i)$ is pulled during phase $n$ })}_{s_i}
\end{align*}
where $(i)$ the index of the arm with reward $\mu_{(i)}$.

Let $T_{\nu, i}$ be the number of times arm $(i)$ has been pulled in total at the end of the epoch where $\nu(t) = \nu$. This means
\[
\sum_{\nu \in N_\nu} \text{ Number of times arm $(i)$ is pulled during phase $n$ } = T_{\nu, i} - T_{\nu - 1, i}
\]

Call $n_{E, (i)}$ the phase at which arm $(i)$ is rejected.
Call $\nu_i$ the epoch where arm $(i)$ is eliminated. This means $n_{E, (i)} \in N_{\nu_i}$.
Call $s_i = \sum_{\nu = i}^{\nu_i} (\mu_{\nu + 1} - \mu_{(i)}) (T_{\nu, i} - T_{\nu-1, i})$.

We have
\begin{align*}
    s_i = \underbrace{(\mu_{\nu_i + 1} - \mu_{(i)})(T_{\nu_i, i} - T_{\nu_i - 1, i})}_{(iii)} + \underbrace{\sum_{\nu=i}^{\nu_i-1} (\mu_{(\nu + 1)} - \nu_{(i)}) ( T_{\nu, i} - T_{\nu - 1, i})}_{(iv)}
\end{align*}

By \Cref{lemma:q} and \Cref{lemma:greedy:arm elimination}, $T_{\nu_i, i} \leq \frac{24 \log(2T^2 K^2)}{p (\mu_{(\nu_i + 1)} - \mu_{(i)})^2}$ so that
\[
(iii) \leq \frac{24 \log(2T^2 K^2)}{p(\mu_{(\nu_i + 1)} - \mu_{(i)})}
\]

We also have
\begin{align*}
    (iv) &\leq (\mu_{(\nu_i)} - \nu_{(i)}) \sum_{\nu=i}^{\nu_i-1}  ( T_{\nu, i} - T_{\nu - 1, i}) \\
    &\leq (\mu_{(\nu_i)} - \nu_{(i)})T_{\nu_i - 1, i}
\end{align*}

By \Cref{lemma:q} and \Cref{lemma:greedy:arm elimination}, $T_{\nu_i-1, i} \leq \frac{24 \log(2T^2 K^2)}{p (\mu_{(\nu_i)} - \mu_{(i)})^2}$ and by \Cref{lemma:greedy:nu}, $T_{\nu_i-1, i} \leq \frac{24 M^2p \log(2K^2T^2)}{\Delta_{\nu_i-1}^2}$ so that 
\begin{align*}
    T_{\nu_i-1, i} &\leq \min(\frac{24 M^2p \log(2K^2T^2)}{\Delta_{\nu_i-1}^2}, \frac{24 \log(2T^2 K^2)}{p (\mu_{(\nu_i)} - \mu_{(i)})^2}) \\
    &\leq \sqrt{\frac{24 M^2p \log(2K^2T^2)}{\Delta_{\nu_i-1}^2} \frac{24 \log(2T^2 K^2)}{p (\mu_{(\nu_i)} - \mu_{(i)})^2}} \\
    &=\frac{24 M \log(2K^2T^2)}{\Delta_{\nu_i-1}(\mu_{(\nu_i)} - \mu_{(i)}) }
\end{align*}
and therefore 
\begin{align*}
    (iv) &\leq \frac{24 M \log(2K^2T^2)}{\Delta_{\nu_i-1}} 
\end{align*}

Then either $\nu_i = \nu^*$ and
\[
s_i \leq \frac{24 \log(2T^2 K^2)}{\mu_{(\nu^* + 1)} - \mu_{(i)}} + \frac{24 M \log(2K^2T^2)}{\Delta_{\nu^*-1}} 
\]
or $\nu_i < \nu^*$ and then, 
\begin{align*}
    s_i &= \sum_{\nu = i}^{\nu_i} (\mu_{\nu + 1} - \mu_{(i)}) (T_{\nu, i} - T_{\nu-1, i}) \\
    &\leq  (\mu_{\nu_i + 1} - \mu_{(i)}) \sum_{\nu = i}^{\nu_i} (T_{\nu, i} - T_{\nu-1, i}) \\
    &\leq (\mu_{\nu_i + 1} - \mu_{(i)}) T_{\nu_i, i} \\
    &\leq \frac{24 M \log(2K^2T^2)}{\Delta_{\nu_i}} \\
    &\leq \frac{24 M \log(2K^2T^2)}{\Delta_{\nu^* - 1}} 
\end{align*}
where at the last line we used again \Cref{lemma:greedy:arm elimination} and \Cref{lemma:greedy:nu}.

So in any case 
\[
(i) \leq \sum_{i=1}^{\nu^*} \frac{24 \log(2T^2 K^2)}{\mu_{(\nu^* + 1)} - \mu_{(i)}} + p\nu^*\frac{24 M \log(2K^2T^2)}{\Delta_{\nu^*-1}} 
\]

  \paragraph{Bounding (ii)} we have
  \begin{align*}
    (ii) &\leq \sum_{\nu=0}^{\nu^*} \sum_{t = t_{\nu - 1} + 1}^{t_{\nu}} p \sum_{j \in \mathcal{E}_{\nu}^*} (\mu_j - \mu_{(\nu)}) \ind{j \notin \mathcal{E}(t)}
  \end{align*}
  We call $u_n = \nu_n - |[K] \setminus \mathcal{K}_n|$ the number of arms put under pressure during phase $n$. We have:
  \begin{align*}
    (ii) &\leq \sum_{\nu=0}^{\nu^*} \sum_{n \in \mathcal{N}_\nu}  p \sum_{j \in \mathcal{E}_{\nu}^*} (\mu_{(j)} - \mu_{(\nu)}) \sum_{t = t_{\nu - 1} + 1}^{t_{\nu}} \ind{j \notin \mathcal{E}(t)} \ind{ \text{ $t$ belong to phase $n$ }}  \\
    &\leq \sum_{\nu=0}^{\nu^*} \sum_{n \in \mathcal{N}_\nu}  p \sum_{j \in \mathcal{E}_{\nu}^*} (\mu_{(j)} - \mu_{(\nu)}) (\text{ Number of times arm $j$ is not pulled during phase $n$})  \\
    &= \sum_{\nu=1}^{\nu^*} \sum_{n \in \mathcal{N}_\nu}  p \sum_{j \in \mathcal{E}_{\nu}^*} (\mu_{(j)} - \mu_{(\nu)}) (\text{ Number of times arm $j$ is not pulled during phase $n$})  \\
    &\leq \sum_{\nu=1}^{\nu^*} \sum_{n \in \mathcal{N}_\nu}  p \sum_{j \in \mathcal{E}_{\nu}^*} (\mu_{(j)} - \mu_{(\nu)}) u_n \ind{n \leq \underbrace{n_{A, j}}_{\text{ Last phase where arm $j$ is not accepted }}} \\
    &\leq \sum_{\nu=1}^{\nu^*} \sum_{n \in \mathcal{N}_\nu}  p \sum_{j \in \mathcal{E}_{\nu}^*} (\mu_{(j)} - \mu_{(\nu)}) \sum_{\nu'=1}^\nu \ind{n \leq \underbrace{n_{E, \nu'}}_{\text{ Last phase where arm $\nu'$ is not rejected}}} \ind{n \leq n_{A, j}} \\
    &= \sum_{\nu'=1}^{\nu^*} p \underbrace{\sum_{\nu=\nu'}^{\nu^*} \sum_{j= \nu + 1}^K \sum_{n \in \mathcal{N}_\nu}  (\mu_{(j)} - \mu_{(\nu)}) \ind{n \leq n_{E, \nu'}} \ind{n \leq n_{A, j}}}_{(A)} \\
  \end{align*}

Call $n_{\nu}$ the last phase before $\nu$ is increased.
We have that 
\begin{align*}
(A) &= \sum_{\nu=\nu'}^{\nu^*-1} \sum_{j= \nu + 1}^K \sum_{n \in \mathcal{N}_\nu}  (\mu_{(j)} - \mu_{(\nu)}) \ind{n \leq n_{E, \nu'}} \ind{n \leq n_{A, j}} \\ &\quad + \sum_{j= \nu^* + 1}^K \sum_{n \in \mathcal{N}_{\nu^*}}  (\mu_{(j)}- \mu_{(\nu^*)}) \ind{n \leq n_{E, \nu'}} \ind{n \leq n_{A, j}} \\
&\leq  \sum_{n \in \NN} \sum_{j= \nu + 1}^K  (\mu_{(j)} - \mu_{(\nu)}) \ind{n \leq n_{A, j}} \ind{n \leq n_{\nu^*-1}} + \sum_{j= \nu^* + 1}^K \sum_{n \in \NN}  (\mu_{(j)} - \mu_{(\nu^*)}) \ind{n \leq n_{E, \nu^*}} \ind{n \leq n_{A, j}}
\end{align*}

Call $\tau_n$ the value of $\tau$ at the end of phase $n$. First notice that
  \begin{align*}
    n \leq n_{A, j} &\implies  q_{n-1} \leq q_{A, j} \tag{$q_{n-1}$: value of $q$ at the end of phase $n-1$} \\
    &\implies  \frac13 p \tau_{n-1} \leq q_{A, j} \tag{By \Cref{lemma:q}} \\
    &\implies \tau_{n-1} \leq \frac{24\log(2 T^2 K^2)}{(\mu_{(j)} - \mu_{(\nu)})^2p} \tag{By \Cref{lemma:greedy:arm acceptation}} \\
&\implies  \mu_{(j)} - \mu_\nu \leq \sqrt{ \frac{24\log(2 T^2 K^2)}{\tau_{n-1}p}} \defeq \delta_n
  \end{align*}

Next, we write:

 \begin{align*}
   \sum_{j= \nu + 1}^K \ind{n \leq n_{A, j}}  &= \text{ Number of arms not yet accepted at phase $n$} \\
   &= \tau_{n} - \tau_{n-1} \\
 \end{align*}

so we have:

\[
(A) \leq \sum_{n \in \NN} (\tau_n - \tau_{n-1}) \delta_n \bigg(  \ind{n \leq n_{\nu^*-1}} +  \ind{n \leq n_{E, \nu^*}} \bigg)
\]
 Note that $\tau_{n} - \tau_{n-1}$ is the number of pulls during phase $n$ which is equal to $|\mathcal{K}_n \setminus \mathcal{A}_n| - u_n$ and therefore equal to the number of arms that should be accepted but are not yet accepted.

 Using the identity  $\sqrt{ \frac{24\log(2 T^2 K^2)}{\tau_{n-1}p}} = \delta_n$, we get
 \begin{align*}
   \tau_{n} - \tau_{n-1} &= \frac{24\log(2 T^2 K^2)}{p} \bigg(\frac1{\delta_{n+1}^2 - \delta_n^2} \bigg) \\
   &= \frac{24\log(2 T^2 K^2)}{p} \bigg(\frac1{\delta_n} + \frac1{\delta_{n+1}} \bigg) \bigg(\frac1{\delta_{n+1}} - \frac1{\delta_n} \bigg) \\
 \end{align*}

Let us now argue that for any $n>= 1$, $\tau_n \leq 2 \tau_{n-1}$.
First, let us notice that
 \begin{align*}
   \tau_{n} - \tau_{n-1} &= | \mathcal{K}_{n-1} \setminus \mathcal{A}_{n-1}| - (\nu_{n-1} - |[K] \setminus \mathcal{K}_{n-1}| \\
                         &= | \mathcal{K}_{n-1}| - |\mathcal{A}_{n-1}| - \nu_{n-1} + K - |\mathcal{K}_{n-1}| \\
                         &= K - |\mathcal{A}_{n-1}| - \nu_{n-1} \\
   &\leq K
 \end{align*}
 Then note that $\tau_0 = K$ (since all arms are active at the first iteration) so that $2\tau_{n-1} \geq \tau_{n-1} + K \geq \tau_n$
 This implies
 \[
\frac{\delta_{n-1}}{\delta_n} = \sqrt{ \frac{\tau_{n}}{\tau_{n-1}} } \leq \sqrt{2}.
 \]
 
We can then write:
 \begin{align*}
   (A) &\leq \frac{24\log(2 T^2 K^2)}{p}(\sqrt{2} + 1) \sum_{n \in \NN} \bigg(\frac1{\delta_{n+1}} - \frac1{\delta_n} \bigg) \bigg(  \ind{n \leq n_{\nu^* - 1}} + \sum_{n \in \NN} \ind{n \leq n_{E, \nu}} \big)
 \end{align*}

We have
 \begin{align*}
   n \leq n_{E, \nu^*} &\implies q_n \leq q_{E, \nu^*} \\
   &\implies \frac13 p \tau_{n-1} \leq  q_{E, \nu^*} \tag{By \Cref{lemma:q}}\\
   &\implies \tau_{n-1} \leq  \frac{24 \log(2 T^2 K^2)}{(\mu_{(\nu^*+1)} - \mu_{(\nu^*)})^2 p} \tag{By \Cref{lemma:greedy:arm elimination}} \\
   &\implies \mu_{(\nu^*+1)} - \mu_{(\nu^*)} \leq  \sqrt{\frac{24 \log(2 T^2 K^2)}{\tau_{n-1} p}} = \delta_n
 \end{align*}
so that
 \[
\mu_{(\nu^*+1)} - \mu_{(\nu)} \leq \delta_{n_{E, \nu}}
 \]
 
Similarly
 \begin{align*}
   n \leq n_{\nu} &\implies q_n \leq q_{\nu} \\
   &\implies \frac13 p \tau_{n-1} \leq q_{\nu} \tag{By \Cref{lemma:q}}\\
   &\implies \tau_{n-1} \leq  \frac{24 M^2 p^2 \log(2 T^2 K^2)}{\Delta_\nu^2 p} \tag{By \Cref{lemma:greedy:nu}} \\
   &\implies \frac{\Delta_{\nu}}{Mp} \leq  \sqrt{\frac{24 \log(2 T^2 K^2)}{\tau_{n-1} p}} = \delta_n
 \end{align*}
so that
 \[
\frac{\Delta_{\nu^* - 1}}{Mp} \leq \frac{\Delta_{\nu}}{Mp} \leq \delta_{n_{\nu}}
 \]
 
 Using again that $\frac1{\delta_n} \leq \sqrt{2} \frac1{\delta_{n-1}}$, we get

 \[
(A) \leq \frac{24\log(2 T^2 K^2)}{p} (\sqrt{2} + 1) \sqrt{2}\frac1{\mu_{(\nu^*+1)} - \mu_{\nu'}} + M p \frac{24\log(2 T^2 K^2)}{p} (\sqrt{2} + 1) \sqrt{2}\frac1{\Delta_{\nu^*-1}}
 \]

 so that
 \begin{align*}
   (ii) \leq \sum_{\nu'=1}^{\nu^*}   \frac{96\log(2 T^2 K^2)}{\mu_{(\nu^*+1)} - \mu_{\nu'}} + \nu^* M p \frac{96\log(2 T^2 K^2)}{\Delta_{\nu^*-1}}
 \end{align*}
 where we used $2 + \sqrt{2} \leq 4$

 From the bound of $(i)$ and $(ii)$, we get:

 \[
R_{\mathcal{E}} \leq \sum_{\nu'=1}^\nu  \frac{120 \log(2 T^2 K^2)}{\mu_{(\nu^*+1)} - \mu_{\nu'}} + \nu^* Mp \frac{120 \log(2K^2T^2)}{\Delta_{\nu^*-1}}
 \]
 
 Since $Mp \leq \frac1{K}$ we have the result.
\end{proof}


It remains to bound $R_M$. Recall that $R_M$ measures the mismatch between the chosen assignment $\Mb(t)$ and the best possible assignment with the same support. Crucially there is no support mismatch and therefore we are in a setting close to the full information setting which allows us to bound $R_M$ by a quantity independent of the horizon $T$.
\begin{lemma}[Bound on $R_M$]
  \label{lemma:rm bound}
  \[
\EE[R_M] \leq 2Mp (\nu^* + 1)^2 + 2MK \frac{3(\nu^* + 1)}{r}
  \]
\end{lemma}
\begin{proof}[Proof of \Cref{lemma:rm bound}]
\label{app:lemma:rm bound}
The proof of \Cref{lemma:rm bound} follows similar techniques as \citet{huangFollowingLeaderFast2017}. 
\begin{align*}
  \EE[R_M] &= \EE[\sum_{t=1}^T \langle \mub, g(\Mb^*_{\mathcal{E}(t)}) - g(\Mb(t)) \rangle] \\
&= \EE[\sum_{t=1}^{(\nu^* + 1)^2} \langle \mub, g(\Mb^*_{\mathcal{E}(t)}) - g(\Mb(t)) \rangle] + \EE[\sum_{t=(\nu^* + 1)^2}^T \langle \mub, g(\Mb^*_{\mathcal{E}(t)}) - g(\Mb(t)) \rangle] \\
      &\leq 2Mp (\nu^* + 1)^2 + \underbrace{\EE[\sum_{t=(\nu^* + 1)^2}^T \langle \mub, g(\Mb^*_{\mathcal{E}(t)}) - g(\Mb(t)) \rangle]}_{(i)} \tag{By \Cref{eq:rbound}}
\end{align*}

Then we write
\begin{align*}
&\langle \mub, g(\Mb^*_{\mathcal{E}(t)}) - g(\Mb(t)) \rangle \leq \EE[\langle \mub - \hat{\mub}, g(\Mb^*_{\mathcal{E}(t)}) - g(\Mb(t)) \rangle] \\
  &\leq 2 Mp \mathbb{E}[ |\mub - \hat{\mub}|_\infty \mathds{1}\{|\mub - \hat{\mub}|_\infty \geq r\}] \tag{By \Cref{eq:rbound} and definition of $r$}
 \\
 &\leq 2Mp \EE[\Big(r\PP\{|\mub - \hat{\mub}|_\infty \geq r\} + \int_r^\infty \PP\{|\mub - \hat{\mub}|_\infty \geq \varepsilon\}d\varepsilon\Big)]\\
&\leq \EE[2Mp \sum_{k \in \mathcal{K}(t)} \Big(r\exp(-2q(t)r^2)+\int_r^\infty\exp(-2q(t)\varepsilon^2) d\varepsilon\Big)]
\end{align*}
so we have
  \begin{align*}
    (i) &\leq \EE[\sum_{t=(\nu^* + 1)^2}^{\infty}   2MpK\Big(r\exp(-2q(t)r^2)+\int_r^\infty\exp(-2q(t)\varepsilon^2) d\varepsilon\Big)]   \\
     &\leq \EE[\sum_{t=(\nu^* + 1)^2}^{\infty}  2MpK \Big(r\exp(-2\frac13 p \tau(t)r^2)+\int_r^\infty\exp(-2\frac13 p \tau(t)\varepsilon^2) d\varepsilon\Big)] \tag{By \Cref{lemma:q}}  \\
       &\leq \sum_{t=(\nu^* + 1)^2}^{\infty}   2MpK \Big(r\exp(-2\frac13 p \max(\frac{t}{\nu^* + 1} - \nu^*, 0)r^2)+\int_r^\infty\exp(-2\frac13 p \max(\frac{t}{\nu^* + 1} - \nu^*, 0)\varepsilon^2) d\varepsilon\Big) \tag{By \Cref{lemma:tau}}  \\
    &\leq \sum_{t=1}^{\infty}  2MpK
    \Big(r\exp(-2\frac13 p \frac{t}{\nu^* + 1}r^2)+\int_r^\infty\exp(-2\frac13 p \frac{t}{\nu^* + 1}\varepsilon^2) d\varepsilon\Big) \\
    &\leq  2MpK
    \Big(r \frac{3 (\nu^* + 1)}{2pr^2}+\int_r^\infty\frac{3 (\nu^* + 1)}{2p\varepsilon^2} d\varepsilon\Big) \\
    &\leq   2MpK\frac{3(\nu^* + 1)}{2p} \Big(\frac{1}{r}+\frac{1}{r}\Big) \\
    &= 2MK \frac{3(\nu^* + 1)}{r}
  \end{align*}
  so that
  \[
\EE[R_M] \leq 2Mp (\nu^* + 1)^2 + 2MK \frac{3(\nu^* + 1)}{r}
  \]

\end{proof}


The upper bound of Cautious Greedy in \Cref{prop:greedy:up} follows by combining the previous lemmas. We use $2Mp (\nu^* + 1)^2 \leq \frac{2MK (\nu^* + 1)}{r}$, $11 (\nu^* + 1) KM \leq \frac{11KM (\nu^* + 1)}{r}$ and $2 \leq \frac{KM (\nu^* + 1)}{r}$ to bound the additive terms.


\subsection{Proof of \Cref{greedy:lb:constant}}
\label{lemma:greedy:lb:constant}

\begin{proof}
We assume $M = 2N + 1$.
Take $m_1 = \frac12$, $m_2 = \frac12 + \Delta$, $\mub_1 = (m_1, m_2)$ and $\mub_2 = (m_2, m_1)$.

\paragraph{Condition on $\Delta$ such that $\Mb^* = (N, N+1)$ if $\mub = \mub_1$ and $\Mb^* = (N+1, N)$ if $\mub = \mub_2$}

Let us first find $\Delta$ such that the optimal assignment is $(N, N + 1)$ when $\mub = \mub_1$ and $(N+1, N)$ when $\mub = \mub_2$.
Assume $\mub = \mub_1$, the reasoning is symmetric for $\mub = \mub_2$. 
We want to find $\Delta$ such that for any $-(N+1) \leq x \leq N$ such that $x\neq 0$:
\begin{equation}
\label{eq:optimality}
    g(N - x) \frac12 + g(N+1 + x) (\frac12 + \Delta) \leq g(N) \frac12 + g(N+1) (\frac12 + \Delta)
\end{equation}

First for $x = N$ we look for $\Delta$ in the form $\Delta = \Ocal(p)$
\begin{align*}
    &g(2N + 1) (\frac12 + \Delta) \leq g(N) \frac12 + g(N+1) (\frac12 + \Delta) \\
    &\iff  (2N + 1)(1 - p)^{2N}(\frac12 + \Delta) \leq N \frac12 + (N+1)(1 - p) (\frac12 + \Delta) \\
    &\impliedby  (2N + 1)(1 - p) (\frac12 + \Delta) \leq N \frac12 + (N+1)(1 - p) (\frac12 + \Delta) \tag{Using $(1 -p)^{2N} \leq (1 - p)$}\\
    &\iff N(1 - p) (\frac12 + \Delta) \leq N \frac12 \\
    &\iff \Delta \leq \frac12(\frac1{1-p} - 1) \\
    &\iff  \Delta \leq \frac{p}{2(1 - p)} \\
\end{align*}
Then if $\Delta \leq \frac{p}{1 - p}$, \Cref{eq:optimality} is satisfied for $x = N$.

For $x = -(N+1)$, the left-hand side of \Cref{eq:optimality} is $g(2N + 1) \frac12 \leq g(2N + 1) (\frac12 + \Delta)$ so if $\Delta \leq \frac{p}{1 - p}$ \Cref{eq:optimality} is satisfied for $x = -(N+1)$.

For $0< x < N$, we have
\begin{align*}
    &g(N - x) \frac12 + g(N+1 + x) (\frac12 + \Delta) \leq g(N) \frac12 + g(N+1) (\frac12 + \Delta) \\
    &\iff (N - x) \frac12 + (N+1+x) (1 - p)^{1 + 2x}(\frac12 + \Delta) \leq N (1 - p)^{x} + (N+1)(1 - p)^{x + 1} (\frac12 + \Delta) \\
    &\impliedby  \bigg(x(1 - p)^{x + 1} \bigg)(\frac12 + \Delta) \leq N (1 - p)^{x} - (N - x) \frac12  \tag{Using $(1-p)^{2x+1} \leq (1 - p)^{x + 1}$}\\
    &\impliedby  \bigg(x(1 - p)^{x + 1} \bigg)(\frac12 + \Delta) \leq \frac{x}{2}  \tag{Using $(1-p)^{x} \geq \frac1{\sqrt{e}} \geq \frac12$ since $x \leq \frac1{-2\log(1-p)}$}\\
    &\impliedby  \bigg(x(1 - p) \bigg)(\frac12 + \Delta) \leq \frac{x}{2} \tag{Using $(1-p)^{x + 1} \leq (1 - p)$}\\
    &\iff \Delta \leq \frac12 (\frac1{1 - p} - 1) \\
    &\iff \Delta \leq \frac{p}{2(1 - p)}
\end{align*}
Therefore if $\Delta \leq p$, \Cref{eq:optimality} is satisfied for $0< x < N$.

For $ -(N + 1) < x < 0$, set $y = -x - 1$ so that $x = -y - 1$ and $ 0 \leq y \leq N$. We can write 
$g(N - x) \frac12 + g(N+1 + x) (\frac12 + \Delta) = g(N + y + 1) \frac12 + g(N - y) (\frac12 + \Delta) < g(N + y + 1) (\frac12 + \Delta)+ g(N - y) \frac12$ which gives the desired inequality for $y = 0$.
For $y > 0$, \Cref{eq:optimality} is satisfied if $\Delta \leq \frac{p}{1 - p}$. Therefore if $\Delta \leq \frac{p}{1 - p}$, \Cref{eq:optimality} is satisfied and therefore, the optimal assignment if $\mub = \mub_1$ is $\Mb^* = (N, N + 1)$.

\paragraph{Computing $r$}
Let us now compute $r$. Assume again $\mub = \mub_1$ and the reasoning is symmetric for $\mub = \mub_2$.
We have $\mathcal{M}_1 \cup \mathcal{M}_0 = \mathcal{M}$ and we know $\Mb^* = (N, N + 1)$ so that $r = \min_{\mub', \argmax_{\Mb \in \Mcal} \langle \mub', g(\Mb) \rangle \neq \Mb^*} \|\mub' - \mub\|_\infty$.
Call $\mub_r = \argmin_{\mub', \argmax_{\Mb \in \Mcal} \langle \mub', g(\Mb) \rangle \neq \Mb^*} \|\mub' - \mub\|_\infty$ and $\Mb_r = \argmax_{\Mb \in \Mcal} \langle \mub_r, g(\Mb) \rangle$.

Since the number of players assigned to an arm increases with the reward of this arm, we have either $\mub_r = \mub + (r_1, -r_1)$ and then $\Mb_r = \Mb^* + (1, -1)$ or $\mub_r = (-r_2, r_2)$ and then $\Mb_r = \Mb^* + (-1, 1)$.

$r_1$ is the minimum value such that
\begin{align*}
    &g(N+1) (\frac12 + r_1) + g(N)(\frac12 + \Delta - r_1) \geq g(N) (\frac12 + r_1) + g(N + 1)(\frac12 + \Delta - r_1)\\
    &\iff (g(N+1) - g(N))(\frac12 + r_1) \geq (g(N+1) - g(N))(\frac12 + \Delta - r_1)
\end{align*}
and therefore $r_1 = \frac{\Delta}{2}$

$r_2$ is the minimum value such that
\begin{align*}
    &g(N-1) (\frac12 - r_2) + g(N + 2)(\frac12 + \Delta + r_2) \geq g(N) (\frac12 - r_2) + g(N + 1)(\frac12 + \Delta + r_2) \\
    &\iff (g(N+2) - g(N+1))(\frac12 + \Delta + r_2) \geq (g(N) - g(N - 1))(\frac12 - r_2) \\
    &\iff r_2(g(N+2) - g(N+1) + g(N) - g(N-1)) \geq (g(N+1) - g(N+2))(\frac12 + \Delta) + (g(N) - g(N-1))\frac12 \\
    &\implies r_2 \geq \frac{(g(N+1) - g(N+2))(\frac12 + \Delta) + (g(N) - g(N-1))\frac12}{2(g(N) - g(N-1))} \\
    &\iff r_2 \geq \frac{((N+1) (1 -p)^2 - (N+2)(1-p)^3)(\frac12 + \Delta) + (N(1-p) - (N-1))\frac12}{ 2(N(1-p) - N-1)} \\
    &\iff r_2 \geq \frac{(1-p)^2((N+2)p - 1)(\frac12 + \Delta) + (1 - Np)\frac12}{2( 1 - Np)} \\
    &\implies r_2 \geq \frac{2p(1 - p)^2\frac12 + (1-p)^2((N+2)p - 1)\Delta}{2(1 - Np)} \tag{Using $(1-p)^2 \leq 1$}\\ 
    &\implies r_2 \geq \frac{\frac14(p - \Delta)}{2(1 - Np)} \tag{Using $(1-p)^2 \geq \frac14$ since $p \leq \frac12$}\\ 
    &\implies r_2 \geq \frac14(p - \Delta) \tag{Using $p \leq \frac1{2N}$}\\ 
\end{align*}
Therefore, we choose $\Delta \leq \frac{p}{6}$ so that $\frac{\Delta}{2} < \frac14(p - \Delta)$ meaning $r = r_1 = \frac{\Delta}{2}$.

\paragraph{Improve the power of the algorithm}
Let $A$ be any algorithm that we run on data $\mub$ such that either $\mub = \mub_1$ or $\mub = \mub_2$ (the choice is made by an adversary).
Let us increase the amount of information available to $A$. $A$ is told that the optimal solution is either $\mub_1$ or $\mub_2$. Furthermore, at each time step, $A$ chooses $\Mb(t)$ and observes a sample from arm $1$ with probability $g(M)$ and similarly for arm 2. However $A$ does not observe the rewards. Note that this problem is simpler than the original problem since in the original problem $A$ observes a sample from arm $k$ with probability $g(M_k(t)) \leq g(M)$.
Therefore, at each time step, $A$ should play either (N, N+1) or (N+1, N) since any other play would lead to a higher regret.

\paragraph{Link with classical 2-arms bandit problem}

With the additional information $A$ can be seen as playing a 2 arm bandits with probabilistic triggered arms: playing arm 1 means playing $\Mb(t) = (N, N+1)$ and playing arm 2 means playing $\Mb(t) = (N + 1, N)$. Call $i^*$ the optimal arm.

We follow the technique used in \cite{wang2017improving} to rewrite a bandit problem with probabilistically triggered arms into a classical bandit problem with well chosen discrete random variables: at each time step $t$, $A$ chooses an arm $i_t \in \{ 1, 2 \}$ and observes $\Xb(t) = (X_{1t}, X_{2t})$ where $X_{it} = 1$ with probability $g(M) \mu_i$, $X_{it} = 0$ with probability $g(M) (1 - \mu_{i})$ and $X_{it} = \perp$ with probability $1 - g(M)$.

However, the regret of $A$ is computed as in the original problem (and this information is known to $A$):
\begin{align*}
    \EE[R_A] &= \EE[\sum_{t=1}^T \ind{i_t \neq i^*} \langle \mub, g(\Mb^*) - g(\Mb(t)) \rangle] \\
    &= \EE[\sum_{t=1}^T \ind{i_t \neq i^*} (\frac12 + \Delta)(g(N+1) - g(N)) + (\frac12)(g(N) - g(N+1))] \\ 
     &=\EE[\sum_{t=1}^T \ind{i_t \neq i^*} (\Delta)(g(N+1) - g(N))
\end{align*}

Then, the rest of the proof is then identical to \cite{mourtadaOptimalityHedgeAlgorithm2019}.
Call for $i=1, 2$, let $\PP_i$ be the joint probability on $(\Xb(1), \dots, \Xb(T))$ when $\mub = \mub_i$.

The regret incurred by $A$ on the worst choice of $\mub$ is higher than the regret incurred by choosing the worst between $\mub_1$ and $\mub_2$.
\begin{align*}
    \EE[R_A] &\geq \max_{i^* \in \{ 1, 2 \}} \EE_{i^*}[\sum_{t=1}^T \ind{i_t \neq i^*} (\Delta)(g(N+1) - g(N))] \\
    &\geq \frac12 \sum_{i^*=1}^2 \EE_{i^*}[\sum_{t=1}^T \ind{i_t \neq i^*} (\Delta)(g(N+1) - g(N))] \\
    &= \frac{\Delta (g(N+1) - g(N))}{2} \sum_{i^*=1}^2 \EE_{i^*}[T - \underbrace{N_{i^*}}_{\defeq  \sum_{t=1}^T  \ind{i_t = i^*}}] \\
    &\geq \frac{\Delta (g(N+1) - g(N))}{2} \frac{T}{2} \sum_{i^*=1}^2 \PP_{i^*}[\frac{T}{2} \geq N_{i^*}] \\
    &\geq \frac{\Delta (g(N+1) - g(N))}{2} \frac{T}{2} (\PP_1(N_1 \geq \frac{T}{2}) + \PP_2(N_2 \geq \frac{T}{2}))
\end{align*}

Then by Bretagnolle–Huber inequality (Th 14.2 in \cite{lattimoreBanditAlgorithms2020}), we have
\[
\PP_1(N_1 \geq \frac{T}{2}) + \PP_2(N_2 \geq \frac{T}{2}) \geq \frac12 \exp(-KL(\PP_1, \PP_2))
\]
where $KL$ is the KL-divergence.

More precisely, we have
\begin{align*}
KL(\PP_1, \PP_2) &\leq T g(M) (KL(\Bcal((\frac12 + \Delta)), \Bcal(\frac12)) + KL(\Bcal((\frac12)), \Bcal(\frac12 + \Delta))) \\
&\leq 4 T g(M) \Delta^2
\end{align*}

and therefore 
\begin{align*}
    \EE[R_A] \geq \frac{\Delta (g(N+1) - g(N))}{2} \frac{T}{4} \exp(-4 T g(M) \Delta^2)
\end{align*}
and since the regret increases with $T$ (see (a)), we can assume without loss of generality that $T =  \floor{\frac{1}{4 g(M) \Delta^2}} \geq \frac{1}{8 g(M) \Delta^2}$ and obtain

\begin{align*}
    \EE[R_A] &\geq \frac{(g(N+1) - g(N))}{64 g(M) \Delta} \exp(-1) \\
    &\geq \frac{(g(N+1) - g(N))}{64 Mp \Delta} \exp(-1) \\
    &= \frac{( (N + 1)(1 - p) - N)}{64 M \Delta} \exp(-1) \\
    &= \frac{(1 - (N + 1)p)}{64 M \Delta} \exp(-1) \\
    &\geq \frac{1}{128 M \Delta} \exp(-1) \tag{Using $p\leq \frac1{2(N+1)}$}\\
\end{align*}
\end{proof}

\subsection{Proof of \Cref{greedy:lb:gap}}\label{lemma:greedy:lb:constant2}
Take $K = \nu^{*} + 2$ arms, $M$ players and $\mub = (\mu_1, \mu_{0}, \mu_{0} + \Delta_{(1)} - \Delta_{(2)}, \dots, \mu_{0} + \Delta_{(1)} - \Delta_{(\nu^{*})}, \mu_{0} + \Delta_{(1)})$.
For simplicity denote $\Delta = \Delta_{(1)}$.

Let us choose $\mu_{1}, \mu_{0}$ and $\Delta$ such that the $\nu^{*} + 1$-st best  assignments are to put $M-1$ player on the first arm and one player on a different arm.

For this we need to ensure the three conditions:
\begin{gather} 
\label{eq:lb:cond1}
 g(M-1) \mu_1 + g(1) (\mu_0 + \Delta) \geq g(M-2) \mu_1 + 2 g(1) (\mu_0 + \Delta)\\
\label{eq:lb:cond2}
g(M-1) \mu_1 + g(1) \mu_0 \geq g(M) \mu_1\\ 
\label{eq:lb:cond3}
 g(M) \mu_1 \geq g(M-2) \mu_1 + g(2) (\mu_0 + \Delta)
\end{gather}

\Cref{eq:lb:cond1} ensures that putting strictly less than $M-1$ players on the first arm is sub-optimal. \Cref{eq:lb:cond2} ensures that putting $M$ players on the first arm is worse than any assignment that puts exactly $M-1$ players on the first arm. \Cref{eq:lb:cond3} ensures that putting stricly less than $M-1$ players on the first arm is worse than putting all players on the first arm.

\Cref{eq:lb:cond1} yields
\begin{align*}
& g(M-1) \mu_1 + g(1) (\mu_0 + \Delta) \geq g(M-2) \mu_1 + 2 g(1) (\mu_0 + \Delta) \\
& \iff (\mu_0 + \Delta) \leq \underbrace{\frac{g(M-1) - g(M-2)}{g(1)}}_{h_1} \mu_1
\end{align*}

\Cref{eq:lb:cond2} yields
\begin{align*}
&g(M-1) \mu_1 + g(1) \mu_0 \geq g(M) \mu_1 \\
& \iff \mu_0 \geq \underbrace{\frac{g(M) - g(M-1)}{g(1)}}_{h_2} \mu_1
\end{align*}

\Cref{eq:lb:cond3} yields
\begin{align*}
&g(M) \mu_1 \geq g(M-2) \mu_1 + g(2) (\mu0 + \Delta) \\
&\iff \underbrace{\frac{g(M) - g(M-2)}{g(2)}}_{h_3} \mu_1 \geq (\mu_0 + \Delta)
\end{align*}

We have $h_1 > h_2$ and 
\begin{align*}
 h_3 &= \frac{g(M) - g(M-1) + g(M-1) - g(M-2)}{g(2)}\mu_1 \\
 &> \frac{2(g(M) - g(M-1))}{2g(1)}\mu_1 \\
 &= h_2.
\end{align*}

We therefore choose $\mu_1 = 1$, $\mu_0 = \frac{h_2 + \min(h_1, h_3)}{2}$ and need $\Delta \leq \frac{\min(h_1, h_3) - h_2}{4}$

Since $g(M) - g(M-1) = p(1 - p)^{M-2}(1 - Mp)$ and 
\begin{align*}
g(M) - g(M-2) &= Mp (1-p)^{M-1} - (M-2)p(1-p)^{M-3} \\
&= p(1-p)^{M-3}(M(1-p)^2 - (M-2)) \\
&= p(1-p)^{M-3}(M(1-2p + p^2) - M + 2) \\
&= p(1-p)^{M-3}(2 - 2Mp + Mp^2) \\
\end{align*}

we get
\begin{align*}
    h_1 - h_2 &= (1 - p)^{M-3}(1 - (M-1)p) - (1-p)^{M-2}(1 - Mp) \\
    &= (1 - p)^{M-3}(1 - (M-1)p - (1 - p)(1 - Mp)) \\
    &= (1 - p)^{M-3}(1 - (M-1)p - (1 - p)(1 - Mp)) \\
    &= (1 - p)^{M-3}(1 - Mp + p - (1 - Mp -p + Mp^2)) \\
    &= (1 - p)^{M-3}(2p - Mp^2) \\
    &\geq (1 - p)^{M-3}p \tag{Using $p \leq \frac{1}{M}$} \\
    &\geq (1 - p)^{M-3}p \tag{Using $p \leq \frac{1}{M}$} \\
    &\geq \frac{p}{M-3} \tag{Using $\min_{x \in [M]} g(x) = p$} \\
\end{align*}

and

\begin{align*}
    h_3 - h_2 &= \frac12(1-p)^{M-4}(2 - 2Mp + Mp^2) - (1-p)^{M-2}(1 - Mp) \\
    &= \frac12(1-p)^{M-4}(2 - 2Mp + Mp^2 - 2(1 - Mp)(1 - p)^2) \\
     &= \frac12(1-p)^{M-4}(2 - 2Mp + Mp^2 - (1 - Mp)(2 - 4p + 2p^2)) \\
     &= \frac12(1-p)^{M-4}(2 - 2Mp + Mp^2 - (2 - 4p + 2p^2 -2Mp +4 Mp^2 -2 Mp^3)) \\
     &= \frac12(1-p)^{M-4}(4p -2p^2 + 2Mp^3 -3Mp^2) \\
     &\geq \frac12(1-p)^{M-4}(4p -(3M + 2)p^2) \\
     &\geq \frac12(1-p)^{M-4}(p) \tag{Using $p \leq \frac1{M + 1}$} \\
     &\geq \frac{p}{2(M-4)} \tag{Using $\min_{x \in [M]} g(x) = p$}\\
\end{align*}
Noting that $2(M-4) \geq M-3 \iff M \geq 5$, we obtain that $\Delta \leq \frac{\min(h_1, h_3) - h_2}{4}$ is implied by $\Delta \leq \frac{p}{8 (M-4)}$.

Let $N_k(T)$ be the number of samples of arm $k+1$ observed by the consistent algorithm $A$. Using arguments similar to  Lai \& Robbins result~\cite{lai1985asymptotically} \footnote{Consider for any sub-optimal arm $k$ the two possibilities $\mub$ and $\mub'$ such that $\mu'_i = \mu_i$ for all $i$ except for $i=k$ where $\mu'_k = \mu_0 + \Delta_1 + \epsilon$ and use the same arguments as in Lai \& Robbins} we can prove that
$$
\liminf_T \frac{\EE[N_k(T)]}{\log(T)} \geq \frac{1}{2\Delta_{(k)}^2}
$$
If $m_t$ denotes the number of players put on arm $k+1$ at stage $t$, then $\EE[N_k(T)]=\sum_{t=1}^T g(m_t)$. Denote by $\Delta_{k}(m)$ the cost of the best assignment with $m > 0$ players on arm $k+1$, i.e.,
\begin{align*}
	\Delta_{k}(m)&:=\Big(g(M-1)\mu_1+g(1)(\mu_0 + \Delta)\Big)-\Big(g(M-m)\mu_1+g(m)(\mu_0 + \Delta - \Delta_{(k)}\Big) \\
		     &\geq  \Big(g(M-m)\mu_1+g(m)(\mu_0 + \Delta)\Big)-\Big(g(M-m)\mu_1+g(m)(\mu_0 + \Delta - \Delta_{(k)}\Big) \\
		     &=  g(m) \Delta_{(k)}
\end{align*}
and $\Delta_{k}(0) = 0$.

Then consider $\mathfrak{C}_{k}$ the cost of the assignment putting the optimal number of players on arm $k+1$ and the rest on arm $1$, under the constraint that arm $k+1$ has been played sufficiently often.
\begin{align}
	\label{lemma:info condition}
\mathfrak{C}_{k} = \min_{m_1,\ldots,m_T: \sum_t g(m_t) \geq \frac{\log(T)}{2\Delta_{(k)}^2}} \sum_{t=1}^T \Delta_{k}(m_t)
\end{align}
It is clear that
\[ 
	\liminf_{T} \frac{\EE[R(T)]}{\log(T)} \geq \liminf_{T} \frac{\sum_{k=1}^{\nu^{*}} \mathfrak{C}_{k}}{\log(T)} 
\]

The solution of \Cref{lemma:info condition} has a specific form: for $t \in [\tau]$, $m_t$ is constant, equal to $m_\tau$, and defined by 
$$
\tau g(m_\tau) \geq \frac{\log(T)}{2\Delta_{(k)}^2}
$$
and $m_t=0$ afterwards (with a cost also equal to zero).

As a consequence, one gets that, for a specific value of $\tau^*$, 
$$\mathfrak{C}_{k}=\tau^* \Delta_{k}(m_{\tau^*})\geq\frac{\log(T)}{2\Delta_{(k)}^2}\frac{\Delta_{k}(m_{\tau^*})}{g(m_{\tau^*})}\geq \frac{\log(T)}{2\Delta_{(k)}}
$$
as $\Delta_{k}(m)\geq g(m)\Delta_{(k)}$.

This implies that, for any consistent algorithm, one must have 
$$
\liminf_T \frac{\EE[R(T)]}{\log(T)} \geq \sum_{\nu=1}^{\nu^{*}} \frac{1}{2\Delta_{(\nu)}}
$$

\section{Arms elimination when rewards are close}
\label{sec:limitations-1}

\begin{lemma}[Necessary conditions for arm elimination]
  \label{lemma:greedy:necessary condition elimination}
  Let $k^* = \argmax_{k \in [K]} \mu_k$ and $\alpha = \frac{Mp}{K}$.
  If $p \leq 0.1$, $\alpha \in (2p, 1)$, and  
  $\min_{k' \in [K]} \frac{\mu_{k'}}{\mu_{k^*}} \geq  1.3 \exp(-\alpha) (1 - \alpha)$,
 then $\nu^* = 0$.
\end{lemma}


\begin{proof}[Proof of \Cref{lemma:greedy:necessary condition elimination}]

  From \cite{bonnefoiMultiArmedBanditLearning2017}, $g$ is concave if $x \leq \frac2{-\log(1 - p)}$ and so this is also the case for $x \leq \frac1{-\log(1 - p)}$.
  Therefore, we have that for any $x \leq \frac1{-\log(1 - p)}, g(x) - g(x-1) \leq g(y) - g(y-1)$ for any $y \leq x$.
  
  Assume $\nu^* > 0$ and consider the optimal policy $\Mb^*$. Then take an eliminated arm $i$ and consider $\Mb'$ constructed from $\Mb^*$ by taking one player from $k^*$ and putting it on the eliminated arm $i$.
  Using $\Mb'$ instead of $\Mb^*$ increase the utility by:
  $G = \mu_i p -\mu_k (g(M_{k^*}^*) - g(M_{k^*}^*-1))$.
  
  Note that $M_{k^*}^* \geq M_{k}$ for any $k \neq k^*$ since $k^*$ is the best arm. In particular $M_{k^*} \geq \frac{M}{K}$ and by the hypothesis on the range of $\alpha$, we have $\frac{M}{K} > 2$.
  Also note that by definition of $\Delta_{max}$, $\mu_i \geq \rho \mu_{k^*}$.

  We can then write :
  \begin{align*}
  G &=\mu_i p -\mu_k (g(M_{k^*}^*) - g(M_{k^*}^*-1))\\
    &\geq  \mu_{k^*}\bigg[\rho p - (g(M_{k^*}^*) - g(M_{k^*}^*-1)) \bigg] \tag{Since  $\mu_i \geq \rho \mu_{k^*}$} \\
  &\geq \mu_{k^*}\bigg[\rho p - (g(\frac{\alpha}{p}) - g(\frac{\alpha}{p} - 1)) \bigg] \tag{By concavity of $g$ and $M_{k^*}^* \geq \frac{M}{K} = \frac{\alpha}{p}$}\\
  &= \mu_{k^*}\bigg[\rho p - p(1-p)^{\frac{\alpha}{p} - 2}(1 - \alpha) \bigg] \\
  &= \mu_{k^*}\bigg[\rho p - \frac{p(1-p)^{\frac{\alpha}{p}}(1 - \alpha)}{(1-p)^2} \bigg] \\
  \end{align*}
  
  The gain is positive if $\rho \geq 1.3 \exp(-\alpha)(1 - \alpha)$ since $\exp(-\alpha) \geq \exp(-\frac{-\log(1-p)}{p}\alpha)=(1-p)^{\frac{\alpha}{p}}$ and $1.3 \geq \frac1{0.9^{2}} \geq \frac1{(1-p)^{2}}$.

Therefore, $\Mb^*$ cannot be an optimal policy. This shows that $\nu^* = 0$.
\end{proof}

\section{Centralized UCB}
\label{app:ucb}
\subsection{Description}
At time $t \in [T]$, for all $k \in [K]$, compute an estimate $\hat{\mu}_k(t)$ of $\mu_k$ using~\eqref{eq:mu_hat} and an upper bound using $\hat{\mu}_k^H(t) = \min(\hat{\mu}_k(t) + \zeta_{k}(t), 1)$
where $\zetab$ is given by \Cref{eq:zeta}
 and take
 \[
   \Mb(t+1) = \argmax_{\Mb \in \mathcal{M}} \langle \hat{\mub}^H(t), g(\Mb) \rangle
 \]
 where $\hat{\mub}^H[k] = \hat{\mu}^H_k$.

The code is given in \Cref{algo:ucb}.

\begin{algorithm}
	\caption{UCB}
  \label{algo:ucb}
	\begin{algorithmic}[1]
		\STATE {\bfseries Input :} $M$ (number of players), $K$ (number of arms), $p$ (probability that a player is active), $T$ (horizon)
    \STATE Initialize estimated rewards: $\hat{\mub}^H = \mathbf{1}$
    \FOR{$t$ from $1$ to $T$}
    \STATE Play $\argmax_{\Mb \in \mathcal{M}} \langle \hat{\mub}^H, g(\Mb) \rangle$
    \STATE Compute $\hat{\mub}$ according to~\eqref{eq:mu_hat}
    \STATE Compute $\zetab$ according to~\eqref{eq:zeta}
    \STATE Set $\hat{\mub}^H = \min(\hat{\mub} + \zetab, \mathbf{1})$
    \ENDFOR
	\end{algorithmic}
\end{algorithm}

\subsection{Analysis}
The next Lemma gives an upper bound on the regret of UCB:
\begin{lemma}[Regret of UCB]
  \label{lemma:regret:ucb}
  The regret of UCB satisfies
\begin{equation}
    \label{eq:regret:ucb2}
\EE[R_{UCB}] \leq  2 \sqrt{2K \log(2T^2K^2) T\min(K, Mp + \frac{K}{T})} + 2
\end{equation}
\end{lemma}
\begin{proof}
Define the GOOD event as in \Cref{lemma:concentration:mu}. 

From \Cref{lemma:cb:good}, we have
$\EE[R_{CUCB}] = \EE[R_{UCB} \ind{GOOD}] + 2$.

Then, under the GOOD event, we have:
\begin{align*}
  R_{CUCB} &= \sum_{t=1}^T \langle \mub, g(\Mb^*) \rangle -  \sum_{t=1}^T \langle \mub, g(\Mb(t)) \rangle\\
&= \sum_{t=1}^T \langle \mub - \hat{\mub}^H(t), g(\Mb^*) \rangle + \langle \hat{\mub}^H(t), g(\Mb^*) - g(\Mb(t)) \rangle +  \langle \hat{\mub}^H(t)- \mub, g(\Mb(t)) \rangle\\
&\leq \langle \hat{\mub}^H(t), g(\Mb^*) - g(\Mb(t)) \rangle +  \langle \hat{\mub}^H(t)- \mub, g(\Mb(t)) \rangle \tag{Since $\hat{\mub}^H \geq \mub$ by the GOOD event}\\
&\leq \sum_{t=1}^T \langle \hat{\mub}^H(t)- \mub, g(\Mb(t)) \rangle \tag{Since $\Mb(t) = \argmax_{\Mb \in \Mcal} \langle \hat{\mub}^H, g(\Mb) \rangle$}\\
  &=  \sum_{k=1}^K \sum_{t=1}^T \min(1, 2 \zeta_{k}(t)) g(M_k(t)) \tag{Since $\mub \geq \max(\hat{\mub} - \zetab, \mathbf{0})$ by the GOOD event}\\
  &=  \sum_{k=1}^K \sum_{t=1}^T \min(1,  \sqrt{2\frac{\log(2 T^2 K^2)}{ T_k(t)}}) (g(M_k(t)) - \eta_k(t) + \eta_k(t)) \tag{*} \\
  &\leq  \underbrace{\sum_{k=1}^K \sum_{t=1}^T (g(M_k(t)) - \eta_k(t))}_{(i)} +
\underbrace{\sum_{k=1}^K \sum_{t=1}^T  \sqrt{2\frac{\log(2 T^2 K^2)}{T_k(t)}} \eta_k(t)}_{(ii)}
\end{align*}
(*) Recall the convention that $\hat{\mu}_k = 1$ if $T_k(t) = 0$. In order to ease the notation, we do not make the distinction and write $\frac1{T_k(t)}$ instead of $\frac{\ind{T_k(t) \neq 0}}{T_k(t)} + \ind{T_k(t) = 0}$.

We have that  $\EE[(i)] = 0$ since 
\begin{align*}
 \EE[g(M_k(t)) - \eta_k(t)] &= \EE[g(M_k(t)] - \EE[\EE[\eta_k(t)|M_k(t)]] \\
 &= \EE[g(M_k(t)] - \EE[g(M_k(t)] \\
 &= 0
\end{align*}

and 
\begin{align*}
(ii) &= \sum_{k=1}^K \sum_{t=1}^T  \sqrt{2\frac{\log(2 T^2 K^2)\eta_k(t)}{T_k(t)}}  \tag{Since $\eta_k(t) = \sqrt{\eta_k(t)}$ as $\eta_k(t) \in \{ 0, 1 \}$} \\
&= \sum_{k=1}^K \sqrt{2 \log(2T^2K^2)} \sum_{t=1}^T \sqrt{\frac{\eta_k(t)}{\sum_{\rho=1}^t \eta_k(\rho)}} \\
&= \sum_{k=1}^K \sqrt{2 \log(2T^2K^2)} \sum_{i=1}^{\max(T_k(T), 1)} \frac1{\sqrt{i}} \tag{Since $\forall \rho \in [t], \eta_k(\rho) \in \{ 0, 1 \}$} \\
&\leq \sum_{k=1}^K 2 \sqrt{2 \log(2T^2K^2) \max(T_k(T), 1)}
\end{align*}

Then we have trivially:
\begin{equation}
\label{ucb:general:bound}
\EE[(ii)] \leq 2K \sqrt{2 \log(2T^2K^2) T}
\end{equation}

Otherwise, we write:
\begin{align*}
\EE[(ii)] &\leq \EE[2 \sqrt{2K \log(2T^2K^2) \sum_{k=1}^K(T_k(T) + \ind{T_k(T) = 0})}] \tag{Using $\sum_{i=1}^K \sqrt{a_i} \leq \sqrt{K \sum_{i=1}^K a_i}$} \\
&\leq 2 \sqrt{2K \log(2T^2K^2) \sum_{k=1}^K(\EE[T_k(T)] + \PP(T_k(T) = 0)}) \tag{By Jensen inequality} \\
&= 2 \sqrt{2K \log(2T^2K^2)\sum_{k=1}^K(\sum_{\rho=1}^T g(M_k(\rho)) + \prod_{\rho=1}^T (1 - g(M_k(\rho)))}) \\
&\leq 2 \sqrt{2K \log(2T^2K^2)\sum_{k=1}^K(\sum_{\rho=1}^T M_k p + 1}) \tag{Since $0 \leq g(M_k) \leq 1$ and $g(M_k) \leq M_k p$}\\ 
&\leq 2 \sqrt{2K \log(2T^2K^2) (T Mp + K)})
\end{align*}
and therefore
\[ 
\EE[(ii)] \leq 2 \sqrt{2K \log(2T^2K^2) T\min(K, Mp + \frac{K}{T})}
\]
so that
\[
 \EE[R_{UCB}] \leq  2 \sqrt{2K \log(2T^2K^2) T\min(K, Mp + \frac{K}{T})}
\]
\end{proof}

$\EE[R_{UCB}] \leq 2 K\sqrt{2\log(2T^2K^2)T}$ also holds in the case where players have different probability of activation $(p_i)_{i \ \in [M]}$. This is shown by following the same proof and stopping at \Cref{ucb:general:bound}.

\section{Solving $\argmax_{\Mcal_{\Ecal}} \langle g(\Mb), \vb \rangle$ via a sequential algorithm}
\label{app:optimal:max}

We want to solve
\begin{equation}
\label{max:problem}
\argmax_{\Mcal_{\Ecal}} \langle g(\Mb), \vb \rangle
\end{equation}

where $\Ecal \subset [K]$.

The sequential algorithm of~\citep[Algorithm 5]{dakdoukMassiveMultiplayerMultiarmed2022} is optimal if $\Ecal = \emptyset$ and $\frac{M p }{1 - p} \leq K$ (Th 4.2).
At each time step, the sequential algorithm chooses a new player to assign to an arm based on some arm-specific criterion that decreases with the number of players assigned to this arm (Lemma 4.2).

Call $a_1, \dots, a_M \in [K]$ the arms chosen by the sequential algorithm for players $1, \dots, M$. The first thing to note is that if the first player is assigned to $a_i$ and then the sequential algorithm is run. The resulting algorithm that we call $A$ reaches the same solution as the sequential algorithm (ignoring the order).

Indeed as adding a player to some arm can only decrease its criterion, the assignment chosen by $A$ is $a_i, a_1, \dots, a_k$ until $a_{k+1} = a_i$. Then everything happens as if the assignment chosen by $A$ was $a_1, \dots, a_{k+1}$ and therefore the rest of the run is the same as the sequential algorithm.

Consider $A^*$ is the algorithm that starts by assigning one player to every arm in $\Ecal$ and then follow the sequential algorithm. Cal $\Ecal'$ the set of arms in $\Ecal$ such that for any arm $k \in \Ecal'$ there exists an index $i$ such that $a_i = k$. Then from the previous argument $A^*$ behaves as if one player was assigned to every arm in $\Ecal'' = \Ecal \setminus \Ecal'$ and then the sequential algorithm is run. But since none of the arms in $\Ecal''$ are equal to $a_1, \dots, a_M$ and again because the arm specific criterion decreases with the number of players, the run of $A^*$ after arms in $\Ecal''$ are assigned one player is $a_1, \dots, a_{M - |\Ecal''|}$ which is the optimal solution with $M - |\Ecal''|$ players. This implies that $A^*$ produces the optimal solution.

\end{document}